\documentclass{article}
\usepackage{kr}
\usepackage{times}
\usepackage{soul}
\usepackage[hidelinks]{hyperref}
\usepackage[utf8]{inputenc}
\usepackage[small]{caption}
\usepackage{graphicx}
\usepackage{amsmath}
\usepackage{amsthm}
\usepackage{booktabs}
\usepackage{algorithm}
\usepackage[noend]{algpseudocode}
\usepackage{amssymb}
\usepackage{thm-restate}
\usepackage{enumitem}
\setlist[itemize]{noitemsep}
\usepackage{multirow}
\usepackage{graphicx}
\usepackage{float}
\usepackage{macros}
\usepackage{xcolor}
\usepackage{hyperref}
\usepackage{etoolbox}
\allowdisplaybreaks
\algrenewcommand\algorithmicindent{0.4cm}

\newtoggle{withappendix}
\toggletrue{withappendix}

\title{On the Correspondence Between Monotonic Max-Sum GNNs and Datalog}

\author{%
    David Tena Cucala$^1$\and
    Bernardo Cuenca Grau$^1$\and
    Boris Motik$^{1}$\and
    Egor V. Kostylev$^2$ \\
    \affiliations
    $^1$ Department of Computer Science, University of Oxford, UK\\
    $^2$ Department of Informatics, University of Oslo, Norway\\
    \emails \{david.tena.cucala, bernardo.cuenca.grau, boris.motik\}@cs.ox.ac.uk, egork@uio.no
}

\begin{document}

\maketitle

\begin{abstract}
Although there has been significant interest in applying machine learning
techniques to structured data, the \emph{expressivity} (i.e., a description of
what can be learned) of such techniques is still poorly understood. In this
paper, we study data transformations based on \emph{graph neural networks}
(GNNs). First, we note that the choice of how a dataset is encoded into a
numeric form processable by a GNN can obscure the characterisation of a model's
expressivity, and we argue that a \emph{canonical} encoding provides an
appropriate basis. Second, we study the expressivity of \emph{monotonic
max-sum} GNNs, which cover a subclass of GNNs with max and sum aggregation
functions. We show that, for each such GNN, one can compute a Datalog program
such that applying the GNN to any dataset produces the same facts as a single
round of application of the program's rules to the dataset. Monotonic max-sum
GNNs can sum an unbounded number of feature vectors which can result in
arbitrarily large feature values, whereas rule application requires only a
bounded number of constants. Hence, our result shows that the unbounded
summation of monotonic max-sum GNNs does not increase their expressive power.
Third, we sharpen our result to the subclass of \emph{monotonic max} GNNs,
which use only the max aggregation function, and identify a corresponding class
of Datalog programs.

\end{abstract}

\section{Introduction}

Data management tasks such as query answering or logical reasoning can be
abstractly seen as transforming an input dataset into an output dataset. A key
aspect of such transformations is their \emph{expressivity}, which is often
established by identifying a logic-based language that realises the same class
of transformations. For example, core aspects of the SQL and SPARQL query
languages have been characterised using fragments of first-order logic
\cite{abiteboul95foundation,DBLP:journals/tods/PerezAG09}, and logical
deduction over RDF datasets has been described using the rule-based language
\emph{Datalog} \cite{owl2profiles}. Such correspondences enable rigorous
understanding and comparison of different data management languages.

Recently, there has been an increasing interest in applying machine learning
techniques to data management tasks. A key benefit is that the desired
transformation between datasets can be induced from examples, rather than
specified explicitly. Many models have been proposed for this purpose, such as
recurrent \cite{DBLP:journals/apin/HolldoblerKS99}, fibring
\cite{DBLP:conf/flairs/BaderGH05}, and feed-forward networks
\cite{DBLP:conf/ijcai/BaderHHW07}, architectures that simulate forward
\cite{DBLP:conf/iclr/DongMLWLZ19,DBLP:journals/corr/abs-1809-02193} and
backward chaining \cite{DBLP:conf/nips/Rocktaschel017}, and architectures for
rule learning \cite{DBLP:conf/nips/YangYC17,DBLP:conf/nips/SadeghianADW19}.
\emph{Graph neural networks} (GNNs) have proved particularly popular since they
can express graph transformations and have been widely applied to link
prediction and node classification tasks in structured datasets
\cite{DBLP:conf/esws/SchlichtkrullKB18,DBLP:conf/semweb/PfluegerCK22,DBLP:conf/nips/LiuGHK21,DBLP:conf/icassp/IoannidisMG19,DBLP:conf/icml/QuBT19,DBLP:conf/icml/YangCS16,DBLP:conf/iclr/KipfW17,DBLP:conf/nips/ZhangC18,DBLP:conf/icml/TeruDH20}.

Characterising the expressivity of ML models for data management has thus
steadily gained importance, and computational logic provides a well-established
methodology: we can describe conditions under which ML-induced models become
equivalent to logical formalisms in the sense that applying the ML model to an
arbitrary dataset produces the same result as applying a specific logical
formula. In a pioneering study, \citea{DBLP:conf/iclr/BarceloKM0RS20} showed
that each GNN-induced transformation expressible in first-order logic is
equivalent to a concept query of the $\mathcal{ALCQ}$ \emph{description logic}
\cite{dl-handbook-2}---a popular KR formalism.
\citea{DBLP:journals/corr/abs-2302-02209} proved an analogous result for a
class of GNNs with a dedicated vertex and colour.
\citea{DBLP:conf/aaai/0001RFHLRG19} showed that GNNs can express certain types
of graph isomorphism tests. \citea{DBLP:journals/ml/SourekZK21} characterised
the expressivity of GNNs using a hybrid language where each Datalog rule is
annotated with a tensor. \citea{tccgkm22explainable-gnn-models} characterised
the expressivity of \emph{monotonic GNNs} (MGNNs), which use the max
aggregation function and require all weights in the matrices to be nonnegative,
in terms of a class of Datalog programs. Finally,
\citea{tccgm22faithful-approaches} characterised the expressivity of the
Neural-LP model of rule learning.

In this paper, we take a next step in the study of the expressivity of
GNN-based transformations of structured data. A key technical challenge can be
summarised as follows. GNNs typically use summation to aggregate feature
vectors of all vertices adjacent to a given vertex in the input graph. The
number of adjacent vertices in the input is unbounded (i.e., there is no a
priori limit on the number of neighbours a vertex can have), and so the
summation result can be unbounded as well; hence, it appears that arbitrarily
many vertices can influence whether a fact is derived. This seems fundamentally
different to reasoning in fragments of first-order logic such as Datalog: the
number of constants that need to be jointly considered in an application of a
Datalog rule is determined by the number of rule variables, and \emph{not} by
the structure of the input dataset. Thus, at first glance, one might expect
GNNs with summation to be fundamentally different from Datalog rules. To shed
light on this issue, we present several novel contributions.

In Section~\ref{sec:encoding} we focus on a key obstacle: to apply a GNN to a
dataset, the latter must be encoded as a graph where each vertex is assigned a
numeric feature vector; but then, the expressivity of the transformation
inevitably depends on the details of the encoding, which obscures the
contribution of the GNN itself. To overcome this, we adopt a \emph{canonical}
encoding, variants of which have already been considered by
\citea{DBLP:conf/esws/SchlichtkrullKB18},
\citea{DBLP:conf/iclr/BarceloKM0RS20}, and
\citea{DBLP:conf/semweb/PfluegerCK22}. We define a GNN to be \emph{equivalent}
to a Datalog program if applying the GNN to any dataset while using the
canonical encoding produces the same facts as applying the program's rules to
the dataset \emph{once} (i.e., without fixpoint iteration). Finally, we observe
that noncanonical encodings by \citea{tccgkm22explainable-gnn-models},
\citea{DBLP:conf/aaai/0001RFHLRG19}, or \citea{DBLP:conf/nips/LiuGHK21} can be
described using well-known extensions of Datalog, and so the expressivity of
transformations based on such encodings can be characterised by composing all
relevant programs.

In Section~\ref{sec:maxsum} we present our main technical contribution. First,
we introduce a class of \emph{monotonic max-sum} GNNs. Similarly to the MGNNs
by \citea{tccgkm22explainable-gnn-models}, monotonic max-sum GNNs require
matrix weights to be be nonnegative; however, they allow for the max or sum
aggregation functions in each network layer, and they place certain
restrictions on the activation and classification functions (ReLU and threshold
functions are allowed). \citea{tccgkm22explainable-gnn-models} showed that the
performance of such GNNs with just max aggregation on tasks such as knowledge
graph completion is on a par with that of other recent approaches. Hence,
monotonic max-sum GNNs are practically relevant, but they also allow their
predictions to be explained using logical proofs. Second, we prove that each
monotonic max-sum GNN is equivalent to a Datalog program of a certain shape
possibly containing inequalities in rule bodies. Strictly speaking, such a
program can be recursive in the sense that the same predicate can occur in both
rule bodies and heads; however, our notion of equivalence does not involve
fixpoint iteration (i.e., the program's rules are applied just once). Thus,
monotonic max-sum GNNs can derive facts with predicates from the input, but
they cannot express true recursive properties such as reachability; moreover,
the ability to produce unbounded feature values does not lead to a fundamental
increase in expressivity. Our equivalence proof is quite different from the
analogous result for MGNNs: when aggregation is limited to just max, the value
of each feature of a vertex clearly depends on only a fixed number of
neighbours of the vertex. Third, we prove that the equivalent Datalog program
can be computed from the GNN itself. This result is interesting because it
requires enumerating potentially infinite sets of real-valued candidate feature
values in a way that guarantees termination. This provides a starting point for
future development of practical techniques for extracting Datalog programs from
monotonic max-sum GNNs.

Finally, in Section~\ref{sec:max} we sharpen our results to \emph{monotonic
max} GNNs, which allow only for max aggregation. We show that, analogously to
MGNNs, each monotonic max GNN is equivalent to a positive Datalog program;
however, we also present a converse result: we identify a class Datalog
programs such that, for each program in the class, there exists an equivalent
monotonic max GNN. In this way, we obtain an exact characterisation of an
interesting class of GNN-based transformations using logical formalisms.

The proofs of all theorems are given in full in
\iftoggle{withappendix}{Appendices~\ref{sex:proofs-maxsum} and
\ref{sex:proofs-max}.}{the extended version of this paper
\cite{max-sum-extendedversion}.}

\section{Preliminaries}\label{sec:preliminaries}

We next recapitulate the basics of Datalog and GNNs.

\myparagraph{Datasets and Datalog.} We fix a signature consisting of countably
infinite, disjoint sets of \emph{predicates} and \emph{constants}. Each
predicate is associated with a nonnegative integer arity. We also consider a
countably infinite set of \emph{variables} that is disjoint with the sets of
predicates and constants.

A \emph{term} is a variable or a constant. An \emph{atom} is of the form
${P(t_1, \dots, t_n)}$ where $P$ is a predicate of arity $n$ and ${t_1, \cdots,
t_n}$ are terms. An \emph{inequality} is an expression of the form ${t_1 \noteq
t_2}$ where $t_1$ and $t_2$ are terms. A \emph{literal} is an atom or an
inequality. A term or a literal is \emph{ground} if it is variable-free. A
\emph{fact} is a ground atom and a \emph{dataset} is a finite set of facts;
thus, datasets cannot contain inequalities. A conjunction $\alpha$ of facts is
true in a dataset $D$, written ${D \models \alpha}$, if ${A \in D}$ for each
fact $A$ in $\alpha$. A ground inequality ${s \noteq t}$ is true if ${s \neq
t}$; for uniformity with facts, we often write ${D \models s \noteq t}$ even
though the truth of ${s \noteq t}$ does not depend on $D$. A (Datalog)
\emph{rule} is of the form \eqref{eq:rule} where ${n \geq 0}$, ${B_1, \dots,
B_n}$ are \emph{body} literals, and $H$ is the \emph{head} atom:
\begin{align}
    B_1 \wedge \dots \wedge B_n \rightarrow H.  \label{eq:rule}
\end{align}
A (Datalog) \emph{program} is a finite set of rules. A \emph{substitution}
$\nu$ is a mapping of finitely many variables to ground terms;
for $\alpha$ a
literal, $\alpha\nu$ is the result of replacing in $\alpha$ each variable $x$
with $\nu(x)$ provided the latter is defined. Each rule $r$ of form
\eqref{eq:rule} defines an \emph{immediate consequence} operator $T_r$ on
datasets: for $D$ a dataset, $T_r(D)$ is the dataset that contains the fact
$H\nu$ for each substitution $\nu$ mapping all variables of $r$ to terms
occurring in $D$ such that ${D \models B_i\nu}$ for each ${1 \leq i \leq n}$.
For $\Prog$ a program, ${T_\Prog(D) = \bigcup_{r \in \Prog} T_r(D)}$.

To simplify the formal treatment, we do not make the usual \emph{safety}
requirement where each variable in a rule must occur in a body atom; in fact,
the body can be empty, which we denote by $\top$. For example, rule ${r = \top
\rightarrow R(x,y)}$ is syntactically valid; moreover, the definition of $T_r$
ensures that $T_r(D)$ contains exactly each fact $R(s,t)$ for all (not
necessarily distinct) terms $s$ and $t$ occurring in $D$.

Conjunctions $\alpha$ and $\beta$ of literals are \emph{equal up to variable
renaming} if there exists a bijective mapping $\nu$ from the set of all
variables of $\alpha$ to the set of all variables of $\beta$ such that
$\alpha\nu$ and $\beta$ contain exactly the same conjuncts; this notion is
extended to rules in the obvious way. A set $S$ \emph{contains} a conjunction
$\alpha$ of literals \emph{up to variable renaming} if there exists ${\beta \in
S}$ such that $\alpha$ and $\beta$ are equal up to variable renaming.

\myparagraph{Graph Neural Networks.} We use $\real$ and $\nnreal$ for the sets
of real and nonnegative real numbers, respectively. Also, we use $\nat$ for the
set of natural numbers, and ${\nnnat = \nat \cup \{ 0 \}}$.

A function ${\sigma : \real \to \real}$ is \emph{monotonically increasing} if
${x < y}$ implies ${\sigma(x) \leq \sigma(y)}$. Function $\sigma$ is
\emph{Boolean} if its range is ${\{ 0, 1 \}}$. Finally, $\sigma$ is
\emph{unbounded} if, for each ${y \in \real}$, there exists ${x \in \real}$
such that ${\sigma(x) > y}$.

A real \emph{multiset} is a function ${S : \real \to \nnnat}$ that assigns to
each ${x \in \real}$ the number of occurrences $S(x)$. Such $S$ is
\emph{finite} if ${S(x) > 0}$ for finitely many ${x \in \real}$; the
\emph{cardinality} of such $S$ is ${|S| = \sum_{x \in \real} S(x)}$; and
$\mathcal{F}(\real)$ is the set of all finite real multisets. We often write a
finite $S$ as a list of possibly repeated real numbers in double-braces
${\llbrace \dots \rrbrace}$. Finally, we treat a set as a multiset where each
element occurs just once.

We consider vectors and matrices over $\real$ and $\nnreal$. For $\vc{v}$ a
vector and $i$ a natural number, $\elt{\vc{v}}{i}$ is the $i$-th element of
$\vc{v}$. We apply scalar functions to vectors element-wise; for example, given
$n$ vectors ${\vc{v}_1, \dots, \vc{v}_n}$ of equal dimension, ${\max \{
\vc{v}_1, \dots, \vc{v}_n \}}$ is the vector whose $i$-th element is equal to
${\max \{ \elt{\vc{v}_1}{i}, \dots, \elt{\vc{v}_n}{i} \}}$.

For $\col$ a finite set of \emph{colours} and ${\ddim \in \nat}$ a
\emph{dimension}, a $(\col,\ddim)$-graph is a tuple ${\gG = \langle \gV, \{
\gE{c} \}_{c \in \col}, \lab \rangle}$ where $\gV$ is a finite set of
\emph{vertices}; for each ${c \in \col}$, ${\gE{c} \subseteq \gV \times \gV}$
is a set of directed \emph{edges}; and \emph{labelling} $\lab$ assigns to each
${v \in \gV}$ a \emph{feature} vector $\vlab{v}{\lab}$ of dimension $\ddim$.
Graph $\gG$ is \emph{symmetric} if ${\langle v,u \rangle \in \gE{c}}$ implies
${\langle u,v \rangle \in \gE{c}}$ for each ${c \in \col}$, and it is
\emph{Boolean} if ${\elt{\vlab{v}{\lab}}{i} \in \{ 0, 1 \}}$ for each ${v \in
\gV}$ and ${i \in \{ 1, \dots, \ddim \}}$. To improve readability, we
abbreviate $\vlab{v}{\lab}$ to just $\vlab{v}{}$ when the labelling function is
clear from the context; analogously, we abbreviate $\vlab{v}{\lab[\ell]}$ to
$\vlab{v}{\ell}$.

A $(\col,\ddim)$-\emph{graph neural network} (\emph{GNN}) $\GNN$ with ${L \geq
1}$ \emph{layers} is a tuple
\begin{align}
    \begin{array}{@{}l@{}}
        \langle \{ \matA{\ell} \}_{1 \leq \ell \leq L}, \{ \matB{\ell}{c} \}_{c \in \col \text{ and } 1 \leq \ell \leq L}, \\[0.5ex]
        \hspace{2.5cm} \{ \bias{\ell} \}_{1 \leq \ell \leq L}, \{ \agg{\ell} \}_{1 \leq \ell \leq L}, \act, \cls \rangle,
    \end{array} \label{eq:GNN}
\end{align}
where, for each ${\ell \in \{ 1, \dots, L \}}$ and ${c \in \col}$,
$\matA{\ell}$ and $\matB{\ell}{c}$ are matrices over $\real$ of dimension
$\ddim[\ell] \times \ddim[\ell-1]$ with ${\ddim[0] = \ddim[L] = \ddim}$,
$\bias{\ell}$ is a vector over $\real$ of dimension $\ddim[\ell]$, ${\agg{\ell}
: \mathcal{F}(\real) \to \real}$ is an \emph{aggregation} function, ${\act :
\real \to \real}$ is an \emph{activation} function, and ${\cls : \real \to \{
0,1 \}}$ is a \emph{classification} function.

Applying $(\col,\ddim)$-GNN $\GNN$ to $(\col,\ddim)$-graph $\gG$ induces the
sequence ${\lab[0], \dots, \lab[L]}$ of vertex labelling functions such that
${\lab[0] = \lab}$ and, for each ${\ell \in \{ 1, \ldots, L \}}$ and ${v \in
V}$, the value of $\vlab{v}{\ell}$ is given by
\begin{align}
\begin{array}{@{}l@{}}
    \vlab{v}{\ell} = \act\Big(\matA{\ell} \vlab{v}{\ell-1} + \\
    \qquad\quad\; \sum\limits_{c \in \col}\! \matB{\ell}{c} \; \agg{\ell} \big(\llbrace \vlab{u}{\ell-1} \mid \langle v,u \rangle \in \gE{c} \rrbrace\big) + \bias{\ell}\Big).
\end{array} \label{eq:GNN-propagation}
\end{align}
The result $\GNN(\gG)$ of applying $\GNN$ to $\gG$ is the Boolean
$(\col,\ddim)$-graph with the same vertices and edges as $\gG$, but where each
vertex ${v \in \gV}$ is labelled by $\cls(\vlab{v}{L})$.

\section{Choosing an Encoding/Decoding Scheme}\label{sec:encoding}

To realise a dataset transformation using a GNN, we must first encode the input
dataset into a graph that can be processed by a GNN, and subsequently decode
the GNN's output back into a dataset. Several encoding/decoding schemes have
been proposed in the literature, and their details differ considerably. As a
result, when characterising GNN-based transformations of datasets using logic,
it can be hard to understand which properties of the characterisation are due
to the chosen encoding/decoding scheme, and which are immanent to the GNN used
to realise the transformation. In this paper we consider primarily the encoding
scheme that straightforwardly converts a dataset into a graph, but we also
discuss how to take other encoding schemes into account.

\subsection{Canonical Encoding/Decoding Scheme}\label{sec:encoding:canonical}

A straightforward way to encode a dataset containing only unary and binary
facts into a Boolean $(\col,\ddim)$-graph is to transform terms into vertices,
use vertex connectivity to describe binary facts, and encode presence of unary
facts in feature vectors. Such encoding/decoding schemes, which we call
\emph{canonical}, have already been widely used in the literature with minor
variations
\cite{DBLP:conf/esws/SchlichtkrullKB18,DBLP:conf/semweb/PfluegerCK22,DBLP:conf/iclr/BarceloKM0RS20}.
They establish a direct syntactic correspondence between datasets and coloured
graphs and are thus a natural starting point for studying the expressivity of
GNNs.

We next describe one such scheme. In particular, we introduce
$(\col,\ddim)$-datasets, which naturally correspond to a large class of
$(\col,\ddim)$-graphs. Our definitions provide the foundation necessary to
formulate our expressivity results in Section~\ref{sec:maxsum}. In Section
\ref{sec:encoding:noncanonical} we discuss how to combine our expressivity
results with more complex encoding schemes.

\begin{definition}
    Let $\col$ be a set of colours and let ${\ddim \in \nat}$ be a dimension. A
    $(\col,\ddim)$-\emph{signature} contains
    \begin{itemize}
        \item a binary predicate $\edg{c}$ for each colour ${c \in \col}$, and

        \item a unary predicate $U_i$ for each ${i \in \{ 1, \dots, \ddim \}}$.
    \end{itemize}
    A $(\col,\ddim)$-\emph{fact} has a predicate from the
    $(\col,\ddim)$-signature, and a $(\col,\ddim)$-\emph{dataset} contains only
    $(\col,\ddim)$-facts.
\end{definition}

We assume that terms occurring in datasets correspond one-to-one to vertices of
coloured graphs---that is, each term $t$ is paired with a unique vertex $v_t$.
This is again without loss of generality since the result of applying a GNN to
a coloured graph does not depend on the identity of vertices, but only on the
graph structure and the feature vectors.

We are now ready to define the canonical GNN-based transformations of
$(\col,\ddim)$-datasets.

\begin{definition}\label{def:canonical}
    The \emph{canonical encoding} $\canenc(D)$ of a $(\col,\ddim)$-dataset $D$
    is the Boolean $(\col,\ddim)$-graph ${\langle \gV, \{ \gE{c} \}_{c \in
    \col}, \lab \rangle}$ defined as follows:
    \begin{itemize}
        \item $\gV$ contains the vertex $v_t$ for each term $t$ occurring in
        $D$;

        \item ${\langle v_t,v_s \rangle \in \gE{c}}$ if ${\edg{c}(t,s) \in D}$
        for each ${c \in \col}$; and

        \item ${\elt{\vc{v}_t}{i} = 1}$ if ${U_i(t) \in D}$, and
        ${\elt{\vc{v}_t}{i} = 0}$ otherwise.
    \end{itemize}
    The \emph{canonical decoding} $\candec(\gG)$ of a Boolean
    $(\col,\ddim)$-graph ${\gG = \langle \gV, \{ \gE{c} \}_{c \in \col}, \lab
    \rangle}$ is the dataset that contains
    \begin{itemize}
        \item the fact $\edg{c}(t,s)$ for each ${\langle v_t, v_s \rangle \in
        \gE{c}}$ and ${c \in \col}$, and

        \item the fact $U_i(t)$ for each ${v_t \in \gV}$ and ${i \in \{ 1,
        \dots, \ddim \}}$ such that ${\elt{\vc{v}_t}{i} = 1}$.
    \end{itemize}
    Each $(\col,\ddim)$-GNN $\GNN$ induces the \emph{canonical} transformation
    $T_\GNN$ on $(\col,\ddim)$-datasets where ${T_\GNN(D) =
    \candec(\GNN(\canenc(D)))}$ for each $(\col,\ddim)$-dataset $D$.
\end{definition}

This encoding neither introduces nor omits any information from the input
dataset, so a $(\col,\ddim)$-dataset $D$ and its canonical encoding
$\canenc(D)$ straightforwardly correspond to one another. Since datasets are
directional, $(\col,\ddim)$-graphs must be directed as well to minimise the
discrepancy between the two representations. The canonical decoding is
analogous to the encoding, and the two are inverse operations on graphs that
are regular as per Definition~\ref{def:regular}.

\begin{definition}\label{def:regular}
    A $(\col,\ddim)$-graph ${\gG = \langle \gV, \{ \gE{c} \}_{c \in \col}, \lab
    \rangle}$ is \emph{regular} if $\gG$ is Boolean and each vertex ${v \in
    \gV}$ either occurs in $\gE{c}$ for some ${c \in \col}$, or
    ${\elt{\vlab{v}{}}{i} = 1}$ for some ${i \in \{ 1, \dots, \ddim \}}$.
\end{definition}

Our canonical encoding produces only regular graphs, and there is a one-to-one
correspondence between $(\col,\ddim)$-datasets and regular
$(\col,\ddim)$-graphs. Our results from the following sections can be
equivalently framed as characterising expressivity of GNN transformations of
regular graphs in terms of Datalog programs. Graphs that are not Boolean do not
correspond to encodings of datasets, so we do not see a natural way to view GNN
transformations over such graphs in terms of logical formalisms. Finally, a
$(\col,\ddim)$-graph $\gG$ that is Boolean but not regular contains `isolated'
vertices that are not connected to any other vertex and are labelled by zeros
only. When such $\gG$ is decoded into a $(\col,\ddim)$-dataset, such `isolated'
vertices do not produce any facts in $\candec(\gG)$ and thus several
non-regular Boolean graphs can produce the same $(\col,\ddim)$-dataset. Note,
however, that each `isolated' zero-labelled vertex is transformed by a GNN in
the same way---that is, the vector labelling the vertex in the GNN's output
does not depend on any other vertices but only on the matrices of the GNN.
Consequently, such vertices are not interesting for our study of GNN
expressivity.

We are now ready to formalise our central notion of equivalence between a GNN
and a Datalog program.

\begin{definition}\label{def:capture-equivalence}
    A $(\col,\ddim)$-GNN $\GNN$ \emph{captures} a rule or a Datalog program
    $\alpha$ if ${T_\alpha(D) \subseteq T_\GNN(D)}$ for each
    $(\col,\ddim)$-dataset $D$. Moreover, $\GNN$ and $\alpha$ are
    \emph{equivalent} if ${T_\GNN(D) = T_\alpha(D)}$ for each
    $(\col,\ddim)$-dataset $D$.
\end{definition}

The key question we address in Sections~\ref{sec:maxsum} and~\ref{sec:max} is
the following: under what conditions is a given $(\col,\ddim)$-GNN $\GNN$
equivalent to a Datalog program, and can this program (at least in principle)
be computed from $\GNN$?

\subsection{Noncanonical Encoding/Decoding Schemes}\label{sec:encoding:noncanonical}

For each $(\col,\ddim)$-dataset $D$, the binary facts of $D$ and $T_\GNN(D)$
coincide, and so applying $T_\GNN$ to $D$ cannot derive any binary facts. To
overcome this limitation, more complex, noncanonical encodings have been
proposed
\cite{tccgkm22explainable-gnn-models,DBLP:conf/aaai/0001RFHLRG19,DBLP:conf/nips/LiuGHK21}.
These introduce vertices representing combinations of several constants so that
facts of higher arity can be encoded in appropriate feature vectors, but there
is no obvious canonical way to achieve this. Expressivity results based on such
encodings are less transparent because it is not obvious which aspects of
expressivity are due to the encoding/decoding scheme and which are immanent to
the GNN itself.

We argue that noncanonical encoding/decoding schemes can often be described by
a pair of programs $\Prog_\enc$ and $\Prog_\dec$, possibly expressed in a
well-known extension of Datalog, which convert an input dataset into a
$(\col,\ddim)$-dataset and vice versa. Thus, given an arbitrary dataset $D$,
the result of applying the end-to-end transformation that uses a GNN $\GNN$ and
the respective encoding/decoding scheme is
$T_{\Prog_\dec}(T_\GNN(T_{\Prog_\enc}(D)))$. Furthermore, if $\GNN$ is
equivalent to a Datalog program $\Prog_\GNN$, then the composition of
$\Prog_\enc$, $\Prog_\GNN$, and $\Prog_\dec$ characterises the end-to-end
transformation. This allows us to clearly separate the contribution of the GNN
from the contributions of the encoding and decoding.

\myparagraph{\citea{tccgkm22explainable-gnn-models}} recently presented a
dataset transformation based on a class of \emph{monotonic} GNNs (MGNNs). Their
approach is applicable to a dataset $D$ that uses unary predicates ${A_1,
\dots, A_\epsilon}$ and binary predicates ${R_{\epsilon+1}, \dots, R_{\ddim}}$,
and $D$ is encoded into a symmetric $(\col,\ddim)$-graph over the set of
colours ${\col = \{ c_1, c_2, c_3, c_4 \}}$. The encoding introduces a vertex
$v_a$ for each constant $a$ in $D$ as well as vertices $v_{a,b}$ and $v_{b,a}$
for each pair of constants $a,b$ occurring together in a binary fact in $D$.
Predicates are assigned fixed positions in vectors so that the value of a
component of a vector labelling a vertex indicates the presence or absence of a
specific fact in $D$. For example, if ${A_i(a) \in D}$, then
$\elt{\vc{v}_a}{i}$ is set to $1$; analogously, if ${R_j(a,b) \not\in D}$ but
$a$ and $b$ occur in $D$ in a binary fact, then $\elt{\vc{v}_{a,b}}{j}$ is set
to $0$. Moreover, the edges of the coloured graph indicate different types of
`connections' between constants; for example, vertices $v_a$ and $v_{a,b}$ are
connected by an edge of colour $c_1$ to indicate that constant $a$ occurs first
in the constant pair $(a,b)$. A variant of this approach was also proposed by
\citea{DBLP:conf/nips/LiuGHK21} in the context of knowledge graph completion.

We next show how to capture this encoding using rules. Note that the encoder
introduces vertices of the form $v_{a,b}$ for pairs of constants $a$ and $b$,
so the encoding program $\Prog_\enc$ requires value invention. This can be
conveniently realised using functional terms. For example, we can represent
vertex $v_{a,b}$ using term $g(a,b)$, and we can represent each vertex of the
form $v_a$ using a term $f(a)$ for uniformity. Applying the encoding program
$\Prog_\enc$ to a dataset thus produces a $(\col,\ddim)$-dataset with
functional terms, which should be processed by the GNN as if they were
constants; for example, the canonical encoding should transform $g(a,b)$ into
vertex $v_{g(a,b)}$. Based on this idea, the encoding program $\Prog_\enc$
contains rule \eqref{rule:l:1} instantiated for each ${i \in \{ 1, \dots,
\epsilon \}}$, and rules \eqref{rule:l:2}--\eqref{rule:c4:2} instantiated for
each ${j \in \{ \epsilon + 1, \dots, \ddim \}}$.
\begin{align}
    A_i(x)      & \rightarrow U_i(f(x))                 \label{rule:l:1}  \\
    R_j(x,y)    & \rightarrow U_j(g(x,y))               \label{rule:l:2}  \\
    R_j(x,y)    & \rightarrow \edg{c_1}(f(x),g(x,y))    \label{rule:c1:1} \\
    R_j(x,y)    & \rightarrow \edg{c_1}(g(x,y),f(x))    \label{rule:c1:2} \\
    R_j(x,y)    & \rightarrow \edg{c_2}(f(y),g(x,y))    \label{rule:c2:1} \\
    R_j(x,y)    & \rightarrow \edg{c_2}(g(x,y),f(y))    \label{rule:c2:2} \\
    R_j(x,y)    & \rightarrow \edg{c_3}(g(x,y),g(y,x))  \label{rule:c3:1} \\
    R_j(x,y)    & \rightarrow \edg{c_3}(g(y,x),g(x,y))  \label{rule:c3:2} \\
    R_j(x,y)    & \rightarrow \edg{c_4}(f(x),f(y))      \label{rule:c4:1} \\
    R_j(x,y)    & \rightarrow \edg{c_4}(f(y),f(x))      \label{rule:c4:2}
\end{align}
Rules \eqref{rule:l:1} and \eqref{rule:l:2} ensure that all unary and binary
facts in the input dataset are encoded as facts of the form $U_i(f(a))$ and
$U_j(g(a,b))$; thus, when these are further transformed into a
$(\col,\ddim)$-graph, the vectors labelling vertices $v_{f(a)}$ and
$v_{g(a,b)}$ encode all input facts of the form $A_i(a)$ and $R_j(a,b)$ for ${i
\in \{ 1, \dots, \epsilon \}}$ and ${j \in \{ \epsilon+1, \dots, \ddim \}}$. In
addition, rules \eqref{rule:c1:1}--\eqref{rule:c4:2} encode the adjacency
relationships between terms: colour $c_1$ connects terms $g(a,b)$ and $f(a)$,
colour $c_2$ connects $g(a,b)$ and $f(b)$, colour $c_3$ connects $g(a,b)$ and
$g(b,a)$, and colour $c_4$ connects terms $f(a)$ and $f(b)$ provided that $a$
and $b$ occur jointly in a binary fact.

Program $\Prog_\dec$ capturing the decoder contains rule \eqref{rule:dec:1}
instantiated for each ${i \in \{ 1, \dots, \epsilon \}}$, as well as rule
\eqref{rule:dec:2} instantiated for each ${j \in \{ \epsilon + 1, \dots, \ddim
\}}$.
\begin{align}
    U_i(f(x))   & \rightarrow A_i(x)    \label{rule:dec:1} \\
    U_j(g(x,y)) & \rightarrow R_j(x,y)  \label{rule:dec:2}
\end{align}
Intuitively, these rules just `read off' the facts from the labels of vertices
such as $v_{f(a)}$ and $v_{g(a,b)}$. The composition of these three programs is
a (function-free) Datalog program.

It is straightforward to show that, for each dataset $D$, the graph obtained by
applying the encoder by \citea{tccgkm22explainable-gnn-models} is isomorphic to
the graph obtained by applying the canonical encoding from
Definition~\ref{def:canonical} to $T_{\Prog_\enc}(D)$ and thus program
$\Prog_\enc$ correctly captures their encoder.

A limitation of this encoding is that the transformation's output can contain a
fact of the form $R(a,b)$ only if the input dataset contains a fact of the form
$S(a,b)$ or $S(b,a)$. Intuitively, the presence of $S(a,b)$ or $S(b,a)$ in the
input ensures that the resulting $(\col,\ddim)$-graph contains a vertex
$v_{g(a,b)}$ for representing binary facts of the form $R(a,b)$. An obvious way
to overcome this limitation is to introduce terms $g(a,b)$ for all constants
$a$ and $b$ occurring in the input, without requiring $a$ and $b$ to occur
jointly in a binary fact. While this increases the expressivity of the
end-to-end transformation, the increase is due to the encoding step, rather
than the GNN. Our framework makes this point clear. For example, we can extend
$\Prog_\enc$ with rules such as \eqref{rule:ext:1}--\eqref{rule:ext:4} and so
on for all other combinations of unary and binary predicates and colours. The
chaining of $\Prog_\enc$, $\Prog_\GNN$, and $\Prog_\dec$ can now capture
different transformations even if $\Prog_\GNN$ remains the same.
\begin{align}
    A_i(x) \wedge A_j(y)    & \rightarrow \edg{c_1}(f(x),g(x,y))    \label{rule:ext:1} \\
    A_i(x) \wedge A_j(y)    & \rightarrow \edg{c_1}(g(x,y),f(x))    \label{rule:ext:2} \\
    R_i(x,z) \wedge A_j(y)  & \rightarrow \edg{c_1}(g(x,y),f(x))    \label{rule:ext:3} \\
    R_i(z,x) \wedge A_j(y)  & \rightarrow \edg{c_1}(g(x,y),f(x))    \label{rule:ext:4}
\end{align}

\myparagraph{\citea{DBLP:conf/aaai/0001RFHLRG19}} introduced $k$-GNNs and
showed them to be more expressive than standard GNNs. The input to a $k$-GNN is
a symmetric $(\col,\ddim[1])$-graph $\gG_1$ without self-loops where $\col$
contains a single colour $c$ and, for each vertex $v$ of $\gG_1$,
${\elt{\vlab{v}{}}{i} = 1}$ for exactly one ${1 \leq i \leq \ddim[1]}$. To
apply a $k$-GNN to $\gG_1$, the latter is transformed into another
$(\col,\ddim[2])$-graph $\gG_2$ that contains one vertex for each set of $k$
distinct vertices of $\gG_1$, and then a standard $(\col,\ddim[2])$-GNN is
applied to $\gG_2$.

We next show that the transformation of $\gG_1$ into $\gG_2$ can be captured by
a program $\Prog_\enc$ that transforms a $(\col,\ddim[1])$-dataset over unary
predicates $A_1,\dots, A_{\ddim[1]}$ and a binary predicate $R$ into a
$(\col,\ddim[2])$-dataset. Thus, the increase in expressivity of $k$-GNNs does
not come from the GNN model itself, but rather from the encoding implicit in
their approach. For readability, we make several simplifying assumptions.
First, while \citea{DBLP:conf/aaai/0001RFHLRG19} consider sets of $k$ distinct
vertices in order to ensure practical scalability, we consider $k$-tuples
instead and limit our presentation to just $k=2$. Second, we consider just the
\emph{local neighbourhood} approach to connecting vertices in $\gG_2$. Finally,
our encoding requires extending Datalog not only with function symbols, but
also with stratified negation-as-failure $\mathsf{not}$
\cite{DBLP:journals/csur/DantsinEGV01}.

Program $\Prog_\enc$ consists of rules \eqref{rule:kGNN:1}--\eqref{rule:kGNN:4}
instantiated for all ${i,j,k,\ell \in \{ 1, \dots, \ddim[1] \}}$.
\begin{align}
    \begin{array}{@{}r@{\;}l@{}}
        A_i(x) \wedge A_j(y) \wedge x \noteq y      & \wedge \\
        A_k(x) \wedge A_\ell(z) \wedge x \noteq z   & \wedge \\
        R(y,z) \wedge y \noteq z                    & \rightarrow \edg{c}(g(x,y),g(x,z))
    \end{array}  \label{rule:kGNN:1} \\
    \begin{array}{@{}r@{\;}l@{}}
        A_i(y) \wedge A_j(x) \wedge y \noteq x      & \wedge \\
        A_k(z) \wedge A_\ell(x) \wedge z \noteq x   & \wedge \\
        R(y,z) \wedge y \noteq z                    & \rightarrow \edg{c}(g(y,x),g(z,x))
    \end{array}  \label{rule:kGNN:2} \\
    \begin{array}{@{}l@{}}
        A_i(x) \wedge A_j(y) \wedge x \noteq y \wedge \mathsf{not}~R(x,y) \\
        \hspace{4.5cm} \rightarrow U_{i,j,0}(g(x,y))
    \end{array} \label{rule:kGNN:3}  \\
    \begin{array}{@{}l@{}}
        A_i(x) \wedge A_j(y) \wedge x \noteq y \wedge R(x,y) \\
        \hspace{4.5cm} \rightarrow U_{i,j,1}(g(x,y))
    \end{array} \label{rule:kGNN:4}
\end{align}
Conjunctions of the form ${A_i(x) \wedge A_j(y) \wedge x \noteq y}$ in these
rules identify pairs of distinct constants $a$ and $b$ (corresponding to the
vertices of $\gG_1$) in the input dataset, and, for each such pair, $g(x,y)$
introduces a term $g(a,b)$ (corresponding to a vertex of $\gG_2$). Rules
\eqref{rule:kGNN:1} and \eqref{rule:kGNN:2} encode the \emph{local
neighbourhood} approach: terms $g(a,b)$ and $g(d,e)$ are connected in $\gG_2$
if either ${a = b}$ and ${d \neq e}$, or ${a \neq b}$ and ${d = e}$, and
additionally the two constants in the inequality are connected in $\gG_1$.
Finally, rules \eqref{rule:kGNN:3} and \eqref{rule:kGNN:4} identify the type of
the subgraph of $\gG_1$ that $a$ and $b$ participate in. Specifically, a fact
of the form $U_{i,j,0}(g(a,b))$ says that $a$ and $b$ are labelled in $\gG_1$
by $A_i$ and $A_j$ respectively, but they are not connected in $\gG_1$. A fact
of the form $U_{i,j,1}(g(a,b))$ is analogous, but with the difference that $a$
and $b$ are connected in $\gG_1$.

\section{GNNs with Max-Sum Aggregation}\label{sec:maxsum}

In this section, we introduce monotonic max-sum GNNs and prove that each such
GNN corresponds to a Datalog program (possibly with inequalities in the rule
bodies) that can be computed from the GNN's definition. Monotonic max-sum GNNs
can use the following aggregation function in all layers, which generalises
both max and sum.

\begin{definition}
    For ${k \in \nnnat \cup \{ \infty \}}$, a finite real multiset ${S \in
    \mathcal{F}(\real)}$, and ${\ell = \min{(k,|S|)}}$, let
    \begin{displaymath}
        \maxsum{k}(S) = \begin{cases}
            0                           & \text{if } \ell = 0, \\[1ex]
            \sum\limits_{i=1}^\ell s_i  & \begin{array}{@{}l@{}}
                                            \text{where } s_1, \dots, s_\ell \text{ are the } \\
                                            \ell \text{ largest numbers of } S. \\
                                          \end{array}
        \end{cases}
    \end{displaymath}
\end{definition}

Each occurrence of a number is counted separately; for example,
${\maxsum{3}(\llbrace 0, 1, 1, 2, 2, 5 \rrbrace) = 9}$ because the three
largest numbers in $S$ are $5$ and the two occurrences of $2$. Also,
$\maxsum{1}$ is equivalent to $\max$, and $\maxsum{\infty}$ is equivalent to
$\suma$; hence, $\maxsum{k}$ generalises both the $\max$ and $\suma$
aggregation functions. While the ability to sum just the $k$ maximal elements
may not be relevant in practice, it will allow us to formalise a key technical
result. We next introduce monotonic max-sum GNNs.

\begin{definition}\label{def:max-sum-GNN}
    A \emph{monotonic max-sum} $(\col,\ddim)$-GNN is a GNN of form
    \eqref{eq:GNN} satisfying the following conditions:
    \begin{itemize}
        \item for each ${\ell \in \{ 1, \dots, L \}}$ and each ${c \in \col}$,
        all elements of matrices $\matA{\ell}$ and $\matB{\ell}{c}$ are
        nonnegative;

        \item for each ${\ell \in \{ 1, \dots, L \}}$, the aggregation function
        $\agg{\ell}$ is $\maxsum{\aggK{\ell}}$ for some $\aggK{\ell} \in \nnnat
        \cup \{ \infty \}$;

        \item the activation function $\act$ is monotonically increasing and
        unbounded, and the range of $\act$ is $\nnreal$; and

        \item the classification function $\cls$ is a step function---that is,
        there exists a \emph{threshold} ${t \in \real}$ such that ${\cls(t') =
        0}$ for each ${t' < t}$, and ${\cls(t') = 1}$ for each ${t' \geq t}$.
    \end{itemize}
\end{definition}

Monotonic max-sum GNNs are closely related to, but incomparable with MGNNs by
\citea{tccgkm22explainable-gnn-models}: MGNNs do not require the activation
function to be unbounded, but they support only the $\max$ aggregation function
in all layers. While ReLU satisfies Definition~\ref{def:max-sum-GNN}, neither
ELU nor the sigmoid function is compatible.

In Section~\ref{sec:maxsum:limiting}, we show that, in each monotonic max-sum
GNN $\GNN$, one can replace each function $\maxsum{\aggK{\ell}}$ where
${\aggK{\ell} = \infty}$ with $\maxsum{\capacity{\ell}}$ for some
${\capacity{\ell} \in \nnnat}$ without changing the canonical transformation
induced by $\GNN$---that is, to apply a GNN to a dataset, we need to consider
only a bounded number of vertices for aggregation. Number $\capacity{\ell}$
depends solely on $\GNN$ (i.e., it is independent of any dataset to which
$\GNN$ is applied) and is called the \emph{capacity} of layer $\ell$. In
Section~\ref{sec:maxsum:equivalence}, we use this result to show that $T_\GNN$
is equivalent to the immediate consequence operator of a Datalog program
$\Prog_\GNN$ that depends only on $\GNN$. Finally, in
Section~\ref{sec:maxsum:enumerating}, we show that the numbers
$\capacity{\ell}$ can be computed from $\GNN$, and hence program $\Prog_\GNN$
is computable. Our objective is to show that extracting $\Prog_\GNN$ from
$\GNN$ is possible in principle, but further work is needed to devise a
practical procedure.

\subsection{Limiting Neighbour Aggregation}\label{sec:maxsum:limiting}

Throughout the rest of Section~\ref{sec:maxsum}, we fix a monotonic max-sum
$(\col,\ddim)$-GNN $\GNN$ of form \eqref{eq:GNN} and dimensions ${\ddim[0],
\dots, \ddim[L]}$ as specified in Section~\ref{sec:preliminaries}, and we fix
${\aggK{1}, \dots, \aggK{L}}$ as the numbers defining the aggregation functions
of $\GNN$. We next show that each ${\aggK{\ell} = \infty}$ can be replaced with
a natural number $\capacity{\ell}$. We first introduce several auxiliary
definitions.

\begin{definition}\label{def:possible-values}
    A \emph{$(\col,\ell)$-multiset family}, where ${0 \leq \ell \leq L}$, is a
    mapping $\Ys$ that assigns to each colour ${c \in \col}$ a finite multiset
    $\Ys[c]$ of vectors of dimension $\ddim[\ell]$.

    For each ${1 \leq \ell \leq L}$, each ${1 \leq i \leq \ddim[\ell]}$, each
    vector $\x$ of dimension $\ddim[\ell-1]$, and each $(\col,\ell-1)$-multiset
    family $\Ys$, let
    \begin{displaymath}
        \Val{\ell}{i}{\x}{\Ys} = \elt{\matA{\ell} \x + \sum_{c \in \col} \matB{\ell}{c} \; \maxsum{\aggK{\ell}}(\Ys[c]) + \bias{\ell}}{i}.
    \end{displaymath}

    Sets $\X{\ell}{i}$ with ${0 \leq \ell \leq L}$ and ${1 \leq i \leq
    \ddim[\ell]}$ are defined by induction on $\ell$ as follows.
    \begin{itemize}
        \item For each ${1 \leq i \leq \ddim[0]}$, let ${\X{0}{i} = \{ 0, 1
        \}}$.

        \item For each ${\ell \geq 1}$ and each ${1 \leq i \leq \ddim[\ell]}$,
        set $\X{\ell}{i}$ is the least set that contains
        ${\act(\Val{\ell}{i}{\x}{\Ys})}$ for each vector $\x$ of dimension
        $\ddim[\ell-1]$ such that ${\elt{\x}{j} \in \X{\ell-1}{j}}$ for each
        $j$, and each $(\col,\ell-1)$-multiset family $\Ys$ such that
        ${\elt{\y}{j} \in \X{\ell-1}{j}}$ for all ${c \in \col}$, ${\y \in
        \Ys[c]}$, and $j$.
    \end{itemize}
\end{definition}

Intuitively, sets $\X{\ell}{i}$ contain all real numbers that can occur in the
$i$-th position of a vector labelling a vertex at layer $\ell$ when $\GNN$ is
applied to a canonical encoding of some $(\col,\ell)$-dataset. Indeed, by the
base case of the definition, $\X{0}{i}$ contains all values that can be
produced by the canonical encoding, and the inductive step considers all
possible ways in which a vector in layer $\ell$ can be computed from vectors in
layer $\ell-1$ using propagation equation \eqref{eq:GNN-propagation}. In the
latter case, a $(\col,\ell)$-multiset family $\Ys$ represents a collection of
possible neighbour vectors, and $\Val{\ell}{i}{\x}{\Ys}$ is the argument of the
activation function used to compute some $\elt{\vlab{v}{\ell}}{i}$.

Note that sets $\X{\ell}{i}$ are nonempty, and they can be infinite. However,
Theorem~\ref{thm:min-exists} shows that $\X{\ell}{i}$ can be enumerated as a
countable, monotonically increasing sequence of numbers. This is important
because it shows that the notion of a least nonzero element of $\X{\ell}{i}$ is
correctly defined. In the following, for each ${\alpha \in \real}$, let
${\Xsub{\ell}{i}{\alpha} = \{ \alpha' \in \X{\ell}{i} \mid \alpha' > \alpha
\}}$.

\begin{restatable}{theorem}{minexists}\label{thm:min-exists}
    Each set $\X{\ell}{i}$ satisfies ${\X{\ell}{i} \subseteq \nnreal}$, and,
    for each ${\alpha \in \real}$, set ${\X{\ell}{i} \setminus
    \Xsub{\ell}{i}{\alpha}}$ is finite.
\end{restatable}

Theorem~\ref{thm:min-exists} ensures that, for each ${\alpha \in \real}$, set
$\Xsub{\ell}{i}{\alpha}$ is either empty or it contains a smallest number
strictly larger than $\alpha$. The proof uses the fact that the activation
function $\sigma$ is unbounded. We are now ready to define the capacity of
$\GNN$.

\begin{definition}
    The \emph{capacity} of each layer $\ell$ of $\GNN$ is defined in
    Algorithm~\ref{alg:capacity}. Moreover, the \emph{capacity} of $\GNN$ is
    defined as ${\capacity{\GNN} = \max\{ \capacity{1}, \dots, \capacity{L}
    \}}$.
\end{definition}

Sets $\X{\ell}{i}$ can be infinite, so Algorithm~\ref{alg:capacity} can perhaps
be better understood as inductively defining sequences of numbers
$\afteract{\ell}$, $\beforeact{\ell}$, $\capacity{\ell}$ and so on. However, in
Section~\ref{sec:maxsum:enumerating} we show that the smallest positive
elements of $\X{\ell}{i}$ can in fact be computed, which justifies our usage of
the term `algorithm'.

\begin{algorithm}[tb!]
\caption{\textsc{capacity}$(\GNN)$}\label{alg:capacity}
\begin{algorithmic}[1]
    \State let $\afteract{L}$ be the threshold of $\cls$
    \For{$\ell$ from $L$ down to $1$}
        \State $\minW{\ell} \defeq$ the least non-zero element of $\matA{\ell}$ and all $\matB{\ell}{c}$
        \State $\minX{\ell} \defeq$ the least non-zero number in $\bigcup_i \X{\ell-1}{i}$
        \If{either $\minW{\ell}$ or $\minX{\ell}$ does not exist}
            \State $\capacity{\ell} \defeq \capacity{\ell-1} \defeq \capacity{1} \defeq 0$
            \State \Return                                                                                                                      \label{alg:capacity:st}
        \EndIf
        \State $\beforeact{\ell} \defeq$ the least natural number such that $\act(\beforeact{\ell}) \geq \afteract{\ell}$
        \State $\minB{\ell} \defeq$ the least element of $\bias{\ell}$
        \State $\capacity{\ell} \defeq \min(\aggK{\ell}, \lceil \frac{\beforeact{\ell} - \minB{\ell}}{\minW{\ell} \cdot \minX{\ell}} \rceil)$   \label{alg:capacity:C-ell}
        \State $\afteract{\ell-1} \defeq \frac{\beforeact{\ell} - \minB{\ell}}{\minW{\ell}}$
    \EndFor
\end{algorithmic}
\end{algorithm}

Theorem~\ref{thm:bounded-equivalence} shows that, in each layer of $\ell$, every
$\aggK{\ell}$ that is larger than $\capacity{\ell}$ can be replaced by
$\capacity{\ell}$ without affecting the result of applying $\GNN$ to any
dataset.

\begin{restatable}{theorem}{boundedequivalence}\label{thm:bounded-equivalence}
    Let $\GNN'$ be the $(\col,\ddim)$-GNN obtained from $\GNN$ by replacing
    $\aggK{\ell}$ with $\capacity{\ell}$ for each ${1 \leq \ell \leq L}$. Then,
    ${T_\GNN(D) = T_{\GNN'}(D)}$ for each $(\col,\ddim)$-dataset $D$.
\end{restatable}

Theorem~\ref{thm:bounded-equivalence} can be intuitively understood as follows.
Let $\vlab{v}{\lab[\ell]}$ and $\vlab{v}{\lab[\ell]'}$ be vectors labelling a
vertex $v$ in layer $\ell$ when $T_\GNN$ and $T_{\GNN'}$ are applied to some
$D$. We prove the theorem by showing that either
${\elt{\vlab{v}{\lab[\ell]}}{i} = \elt{\vlab{v}{\lab[\ell]'}}{i}}$ or
${\elt{\vlab{v}{\lab[\ell]}}{i} > \elt{\vlab{v}{\lab[\ell]'}}{i} \geq
\afteract{\ell}}$ for each layer ${\ell \geq \ell_{\sta}}$, where $\ell_{\sta}$
is either the layer where Algorithm \ref{alg:capacity} performs an early return
(via line \ref{alg:capacity:st}) or $0$ if this does not happen. Indeed, assume
that ${\cls(\elt{\vlab{v}{\lab[L]}}{i}) = 1}$ for some $v$. If $\matA{L}$ and
all $\matB{L}{c}$ contain only zeros, or if all $\X{L}{i}$ contain only zeros,
then $L = \ell_{\sta}$; no neighbours of $v$ are needed so we can set all
$\capacity{\ell}$ to $0$ and the equality above holds. Otherwise, $\cls$ is a
threshold function, so ${\elt{\vlab{v}{\lab[L]}}{i} \geq \afteract{L}}$ holds
for $\afteract{L}$ the threshold of $\cls$, and so the argument to the
activation function when computing $\elt{\vlab{v}{\lab[L]}}{i}$ is at least
$\beforeact{L}$. Moreover, $\elt{\vlab{v}{\lab[L]}}{i}$ is produced from
$\elt{\vlab{v}{\lab[L-1]}}{i}$ and the values of $\elt{\vlab{u}{\lab[L-1]}}{j}$
where $u$ ranges over the neighbours of $v$. If we assume that
${\elt{\vlab{v}{\lab[L-1]}}{i} = 0}$ and that $\minX{\ell}$ is the least
nonzero value that each $u$ can contribute to $\elt{\vlab{v}{\lab[L]}}{i}$, it
suffices to have at least ${\lceil \frac{\beforeact{\ell} -
\minB{\ell}}{\minW{\ell} \cdot \minX{\ell}} \rceil}$ nonzero neighbours to
reach $\beforeact{L}$. Thus, we can replace $\aggK{\ell}$ with this number
whenever this number is smaller than $\aggK{\ell}$; in contrast, if
$\aggK{\ell}$ is smaller, we need to keep $\aggK{\ell}$ so that $\GNN'$ does
not derive any new consequences. Finally, $\afteract{L-1}$ is the value of
$\elt{\vlab{v}{\lab[L-1]}}{i}$ in layer $L-1$ to which we can apply analogous
reasoning.

\subsection{Equivalence with Datalog Programs}\label{sec:maxsum:equivalence}

We next show that there exists a Datalog program $\Prog_\GNN$ that is
equivalent to $\GNN$ in the sense described in Definition
\ref{def:capture-equivalence}. Towards this goal, in
Definition~\ref{def:tree-like} we capture the syntactic structure of the rules
in $\Prog_\GNN$ as rules of form \eqref{eq:ruleInequality} where $\varphi$ is a
\emph{tree-like} formula for $x$. To understand the intuition, assume that we
construct from $\varphi$ a graph whose vertices are the variables in $\varphi$,
and where a directed edge from $x$ to $y$ is introduced for each $\edg{c}(x,y)$
in $\varphi$; then, such graph must be a directed tree. Moreover, if variable
$x$ has children $y_1$ and $y_2$ in this graph, then $\varphi$ is allowed to
contain inequalities of the form ${y_1 \noteq y_2}$, which provide $\varphi$
with a limited capability for counting; for example, formula ${\edg{c}(x,y_1)
\wedge \edg{c}(x,y_2) \wedge y_1 \noteq y_2}$ is true precisely for those
values of $x$ that are connected via the $\edg{c}$ predicate to at least two
distinct constants. We also introduce intuitive notions of a \emph{fan-out}
(i.e., the number of children) and \emph{depth} of a variable. Tree-like
formulas contain all concepts of the $\mathcal{ALCQ}$ description logic
\cite{dl-handbook-2} constructed from $\top$, atomic concepts, and concepts of
the form ${\geq n R.C}$ and ${C_1 \sqcap C_2}$; however, our definition also
allows for formulas such as ${\edg{c}(x,y_1) \wedge \edg{c}(x,y_2) \wedge
U(y_1) \wedge y_1 \noteq y_2}$, which do not correspond to the translation of
$\mathcal{ALCQ}$ concepts.

\begin{definition}\label{def:tree-like}
    A \emph{tree-like formula for a variable} is defined inductively as follows.
    \begin{itemize}
        \item For each variable $x$, formula $\top$ is tree-like for $x$.

        \item For each variable $x$ and each unary predicate $U$, atom $U(x)$
        is tree-like for $x$.

        \item For each variable $x$ and all tree-like formulas $\varphi_1$ and
        $\varphi_2$ for $x$ that share no variables other than $x$, formula
        ${\varphi_1 \wedge \varphi_2}$ is tree-like for $x$.

        \item For each variable $x$, each binary predicate $\edg{c}$, and all
        tree-like formulas ${\varphi_1, \dots, \varphi_n}$ for distinct
        variables ${y_1, \dots, y_n}$ where no $\varphi_i$ contains $x$ and no
        $\varphi_i$ and $\varphi_j$ with ${i \neq j}$ share a variable, formula
        \eqref{eq:existential-counting} is tree-like for $x$.
        \begin{align}
            \bigwedge_{i=1}^n \Big( \edg{c}(x,y_i) \wedge \varphi_i \Big) \wedge \bigwedge_{1 \leq i < j \leq n} y_i \noteq y_j \label{eq:existential-counting}
        \end{align}
    \end{itemize}

    Let $\varphi$ be a tree-like formula and let $x$ be a variable in
    $\varphi$. The \emph{fan-out} of $x$ in $\varphi$ is the number of distinct
    variables $y_i$ for which $\edg{c}(x,y_i)$ is a conjunct of $\varphi$. The
    \emph{depth} of $x$ is the maximal $n$ for which there exist variables
    ${x_0, \dots, x_n}$ and predicates ${\edg{c_1}, \dots, \edg{c_n}}$ such
    that ${x_n = x}$ and $\edg{c_i}(x_{i-1},x_i)$ is a conjunct of $\varphi$
    for each ${1 \leq i \leq n}$. The \emph{depth} of $\varphi$ is the maximum
    depth of a variable in $\varphi$.

    For $d$ and $f$ natural numbers, a tree-like formula $\varphi$ is
    \emph{$(d,f)$-tree-like} if, for each variable $x$ in $\varphi$, the depth
    $i$ of $x$ is at most $d$ and the fan-out of $x$ is at most $f(d-i)$.
    Moreover, a Datalog rule is \emph{$(d,f)$-tree-like} if it is of form
    \eqref{eq:ruleInequality}, where $\varphi$ is a $(d,f)$-tree-like formula
    for $x$.
    \begin{align}
        \varphi \rightarrow U(x) \label{eq:ruleInequality}
    \end{align}
\end{definition}

Note that $\varphi$ is allowed to be $\top$ in a rule of
form~\eqref{eq:ruleInequality}; for example, ${\top \rightarrow U(x)}$ is a
valid $(0,0)$-tree-like rule. As explained in Section~\ref{sec:preliminaries},
when applied to a dataset $D$, such a rule derives $U(t)$ for each term $t$
occurring in $D$.

Now let ${\ddim[\GNN] = \max(\ddim[0], \dots, \ddim[L])}$. To construct
$\Prog_\GNN$, we proceed as follows: we compute ${f = |\col| \cdot \ddim[\GNN]
\cdot \capacity{\GNN}}$, we enumerate all $(L,f)$-tree-like rules (up to
variable renaming), and we add to $\Prog_\GNN$ each such rule that is captured
by $\GNN$. Lemma~\ref{lem:captured-rule-test} shows that this latter test can,
at least in principle, be operationalised. In particular, to test whether a
rule ${\varphi \rightarrow U(x)}$ with $n$ variables is captured by $\GNN$, we
consider each possible dataset $D$ obtained from the atoms of $\varphi$ by
substituting the variables with up to $n$ distinct constants, and we check
whether applying $\GNN$ to $D$ derives the analogously instantiated rule head;
if this is the case for all such $D$, then the rule is captured by $\GNN$.
\citea{tccgkm22explainable-gnn-models} used a similar test for MGNNs, but their
approach was simpler since it did not need to support inequalities.
Theorem~\ref{thm:equivalence} then shows that program $\Prog_\GNN$ is indeed
equivalent to $\GNN$.

\begin{restatable}{lemma}{capturedruletest}\label{lem:captured-rule-test}
    Let $r$ be a constant-free Datalog rule with head $H$, let $V$ be the set
    of variables in $r$, and let $A$ be the set of body atoms of $r$. Then,
    $\GNN$ captures $r$ if and only if ${H\nu \in T_\GNN(A\nu)}$ for each
    substitution ${\nu : V \to S}$ such that ${H\nu \in T_r(A\nu)}$, where $S$
    is a set of $|V|$ distinct constants.
\end{restatable}

\begin{restatable}{theorem}{equivalence}\label{thm:equivalence}
    Let $\Prog_\GNN$ be the Datalog program containing, up to variable
    renaming, each ${(L,|\col| \cdot \ddim[\GNN] \cdot
    \capacity{\GNN})}$-tree-like rule captured by $\GNN$, where ${\ddim[\GNN] =
    \max(\ddim[0], \dots, \ddim[L])}$. Then, $\GNN$ and $\Prog_\GNN$ are
    equivalent.
\end{restatable}

To understand this result intuitively, assume that $\GNN$ is applied to a
dataset $D$. The fact that all rules of $\Prog_\GNN$ are captured by $\GNN$
clearly implies $T_{\Prog_\GNN}(D) \subseteq T_{\GNN}(D)$. Furthermore, by
equation \eqref{eq:GNN-propagation}, the value of $\elt{\vlab{v}{L}}{i}$ for
some $i$ is computed from the values of $\elt{\vlab{v}{L-1}}{i}$ and
$\elt{\vlab{u}{L-1}}{j}$ for $k \leq \capacity{L}$ distinct neighbours $u$ of
$v$ per colour and position; but then, if $t$ and $s$ are terms represented by
$v$ and $u$, respectively, the canonical encoding ensures ${\edg{c}(t,s) \in
D}$ for some ${c \in \col}$. Also, $\elt{\vlab{u}{L-1}}{j}$ are computed using
the neighbours of $u$ and so on. Hence, each term $w$ in $D$ that can possibly
influence $\vlab{v}{L}$ must be connected in $D$ to $t$ by at most $L$ such
facts, so all relevant neighbours of $t$ can be selected by a $(d,f)$-tree-like
formula. The inequalities can be used to check for the existence of at least
$k$ distinct neighbours of $t$ in $D$. Now let $D'$ be the subset of $D$
containing precisely the facts that contribute to the value of
$\elt{\vlab{v}{L}}{i}$. We can unfold $D'$ into another tree-like dataset $D''$
that corresponds to the body of an instantiated tree-like rule $r$. Since the
elements of all $\matA{\ell}$ and $\matB{\ell}{c}$ are nonnegative, applying
$\GNN$ to $D$ and $D''$ derives the same value for
$\cls(\elt{\vlab{v}{L}}{i})$. If this value is $1$, then applying the rule $r$
to $D$ produces the same fact as $\GNN$. Furthermore, by definition, $\GNN$
captures $r$ and so ${r \in \Prog_\GNN}$. Thus, $T_{\Prog_\GNN}(D)$ contains
all facts derived by $\GNN$ on $D$.

\subsection{Enumerating Sets $\X{\ell}{i}$}\label{sec:maxsum:enumerating}

The results we presented thus far show that program $\Prog_\GNN$ exists, but it
is not yet clear that $\Prog_\GNN$ is computable: the definition of
$\capacity{\ell}$ in Algorithm~\ref{alg:capacity} uses sets $\X{\ell}{i}$,
which can be infinite. We next show that each $\X{\ell}{i}$ can be enumerated
algorithmically using function $\Next{\ell}{i}{\alpha}$ from
Algorithm~\ref{alg:next} as follows: for $\alpha$ a special symbol $\symstart$,
function $\Next{\ell}{i}{\symstart}$ returns the smallest element of
$\X{\ell}{i}$; moreover, for ${\alpha \in \real}$, function
$\Next{\ell}{i}{\alpha}$ returns the smallest element of
$\Xsub{\ell}{i}{\alpha}$ if ${\Xsub{\ell}{i}{\alpha} \neq \emptyset}$, or
$\symend$ otherwise. For example, $\Next{\ell}{i}{0}$ returns the smallest
nonzero element of $\X{\ell}{i}$, if one exists.

\begin{algorithm}[tb!]
\caption{$\Next{\ell}{i}{\alpha}$}\label{alg:next}
\begin{algorithmic}[1]
    \If{$\ell = 0$}                                                                                         \label{alg:next:ell0:condition}
 	    \If{$\alpha = \symstart$ or $\alpha < 0$}
		    \Return $0$                                                                                     \label{alg:next:ell0:first}
        \ElsIf{$\alpha < 1$}
		    \Return $1$                                                                                     \label{alg:next:ell0:second}
        \Else\;
 		    \Return $\symend$                                                                               \label{alg:next:ell0:third}
 		\EndIf
	\EndIf
    \State let $\Ys_\emptyset$ be such that $\Ys[c]_\emptyset = \emptyset$ for each $c \in \col$            \label{alg:next:init:Y}
	\State $z \defeq \Val{\ell}{i}{\Start{\ell}}{\Ys_\emptyset}$                                            \label{alg:next:init:z}
    \If{$\alpha = \symstart$}
        \Return $\act(z)$                                                                                   \label{alg:next:first}
    \EndIf
	\State $F \defeq \{ \triple{\Start{\ell}}{\Ys_\emptyset}{z} \}$                                         \label{alg:next:init:frontier}
    \While{$F \neq \emptyset$}                                                                              \label{alg:next:loop}
        \State choose and remove $\triple{\x}{\Ys}{z}$ in $F$ with least $z$                                \label{alg:next:mintriple}
        \If{$\act(z) > \alpha$}
            \Return $\act(z)$                                                                               \label{alg:next:return}
        \EndIf
		\For{$\x' \in \Expand{\ell}{\x}$}                                                                   \label{alg:next:x:expand}
			\State $z' \defeq \Val{\ell}{i}{\x'}{\Ys}$                                                      \label{alg:next:x:evaluate}
			\If{$z' > z$}
			    add $\triple{\x'}{\Ys}{z'}$ to $F$                                                          \label{alg:next:x:add}
			\EndIf
	  	\EndFor                                                                                             \label{alg:next:x:expand:end}
		\For{$c \in \col$}                                                                                  \label{alg:next:c}
            \For{$\y \in \Ys[c]$ and $\y' \in \Expand{\ell}{\y}$}                                           \label{alg:next:c:y:expand}
                \State $\Yps \defeq \Ys$ and $\Yps[c] \defeq (\Yps[c] \setminus \{ \y \}) \cup \{ \y' \}$   \label{alg:next:c:y:update}
                \State $z' \defeq \Val{\ell}{i}{\x}{\Yps}$                                                  \label{alg:next:c:y:evaluate}
                \If{$z' > z$}
		    	    add $\triple{\x}{\Yps}{z'}$ to $F$                                                      \label{alg:next:c:y:add}
		 	    \EndIf
            \EndFor                                                                                         \label{alg:next:c:y:expand:end}
            \If{$\Start{\ell}$ contains a nonzero}                                                          \label{alg:next:c:numbers}
                \State $V \defeq \{ \Start{\ell} \}$                                                        \label{alg:next:c:numbers:all0}
            \Else
                \State $V \defeq \Expand{\ell}{\Start{\ell}}$                                               \label{alg:next:c:numbers:expand}
            \EndIf
            \For{$\y' \in V$}
                \State $\Yps \defeq \Ys$ and $\Yps[c] \defeq \Yps[c] \cup \{ \y' \}$                        \label{alg:next:c:numbers:update}
			    \State $z' \defeq \Val{\ell}{i}{\x}{\Yps}$                                                  \label{alg:next:c:numbers:evaluate}
	    		\If{$z' > z$}
	    		    add $\triple{\x}{\Yps}{z'}$ to $F$                                                      \label{alg:next:c:numbers:add}
				\EndIf
            \EndFor                                                                                         \label{alg:next:c:numbers:end}
		\EndFor
    \EndWhile
    \State \Return $\symend$                                                                                \label{alg:next:return:final}
    \vspace{0.17cm}
    \Function{$\mathsf{Start}$}{$\ell$}
        \State \Return the vector $\x$  of dimension $\ddim[\ell-1]$ where
        \Statex \qquad\quad $\elt{\x}{j} = \Next{\ell-1}{j}{\symstart}$ for $1 \leq j \leq \ddim[\ell-1]$
    \EndFunction
    \vspace{0.17cm}
    \Function{$\mathsf{Expand}$}{$\ell,\vc{v}$}
        \State $V \defeq \emptyset$
        \For{$1 \leq j \leq \ddim[\ell-1]$}
            \State $v' \defeq \Next{\ell-1}{j}{\elt{\vc{v}}{j}}$
    		\If{$v' \neq \symend$}
                $V \defeq V \cup \{ \vc{v}[j \gets v'] \}$
            \EndIf
        \EndFor
        \State \Return $V$
    \EndFunction
\end{algorithmic}
\end{algorithm}

In the presentation of Algorithm~\ref{alg:next}, we use the following notation:
for $\vc{x}$ a vector, $j$ an index, and $v$ a real number, $\vc{x}[j \gets v]$
is the vector obtained from $\vc{x}$ by replacing its $j$-th component with
$v$. The algorithm is based on the observation that, since $\matA{\ell}$ and
$\matB{\ell}{c}$ contain only nonnegative elements, and the activation function
is monotonically increasing, we can enumerate the values computed by equation
\eqref{eq:GNN-propagation} in some $\vlab{v}{\ell}$ in a monotonically
increasing fashion. To achieve this, the algorithm maintains a \emph{frontier}
$F$ of triples $\triple{\x}{\Ys}{z}$, each describing one way to compute a
value of $\elt{\vlab{v}{\ell}}{i}$: vector $\x$ reflects the values of
$\elt{\vlab{v}{\ell-1}}{i}$, the $(\col,\ell-1)$-multiset family $\Ys$
describes multisets $\Ys[c]$ reflecting the values of
$\elt{\vlab{u}{\ell-1}}{i}$, and $z$ is $\Val{\ell}{i}{\x}{\Ys}$---that is, the
argument to the activation function when computing $\elt{\vlab{v}{\ell}}{i}$.
The starting point for the exploration (line~\ref{alg:next:init:frontier}) is
provided by $\Start{\ell}$, which returns $\vlab{v}{\ell}$ for a vertex $v$
with no neighbours. To enumerate all candidate values for
$\elt{\vlab{v}{\ell}}{i}$ in an increasing order, the algorithm selects a
triple in the frontier with the smallest $z$ (line~\ref{alg:next:mintriple}),
and considers ways to modify $\x$ or $\Ys$ that increase $z$; each such
combination is added to the frontier (lines \ref{alg:next:x:add},
\ref{alg:next:c:y:add}, and \ref{alg:next:c:numbers:add}). Modifications
involve replacing some component of $\x$ with the next component (lines
\ref{alg:next:x:expand}--\ref{alg:next:x:expand:end}), choosing some ${\y \in
\Ys[c]}$ for some ${c \in \col}$ and replacing some component of $\y$ with the
next component (lines
\ref{alg:next:c:y:expand}--\ref{alg:next:c:y:expand:end}), or expanding some
$\Ys[c]$ with an additional vector (lines
\ref{alg:next:c:numbers}--\ref{alg:next:c:numbers:end}). In the latter case, if
$\Start{\ell}$ contains just zeros, then adding $\Start{\ell}$ to $\Ys[c]$ is
not going to change the computed value of $z$ so the algorithm considers
vectors obtained by expanding $\Start{\ell}$ in order to allow $z$ to increase.
This process produces values of $z$ in an increasing order and it guarantees
that ${\sigma(z) \in \X{\ell}{i}}$. If ${\alpha = \symstart}$, the algorithm
stops when the first such value is produced (line~\ref{alg:next:first}). For
${\alpha \in \real}$, Theorem~\ref{thm:min-exists} guarantees that set
${\X{\ell}{i} \setminus \Xsub{\ell}{i}{\alpha}}$ is finite; since $F$ is
extended only if the value of $z$ increases, either $F$ eventually becomes
empty or $\sigma(z)$ exceeds $\alpha$ so the algorithm terminates
(line~\ref{alg:next:return} or~\ref{alg:next:return:final}).
Theorem~\ref{thm:alg-capacity} captures the formal properties of the algorithm.

\begin{restatable}{theorem}{algcapacity}\label{thm:alg-capacity}
    Algorithm~\ref{alg:next} terminates on all inputs. Moreover, for ${0 \leq
    \ell \leq L}$ and ${1 \leq i \leq \ddim[\ell]}$,
    \begin{itemize}
        \item $\Next{\ell}{i}{\symstart}$ returns the smallest element of
        $\X{\ell}{i}$, and

        \item for each ${\alpha \in \real}$, $\Next{\ell}{i}{\alpha}$ returns
        $\symend$ if ${\Xsub{\ell}{i}{\alpha} = \emptyset}$, and otherwise it
        returns the smallest element of $\Xsub{\ell}{i}{\alpha}$.
    \end{itemize}
\end{restatable}

The complexity of Algorithm~\ref{thm:alg-capacity} depends on the number of
recursive calls to $\mathsf{Next}$, which in turn depends on the matrices of
$\GNN$. We leave investigating this issue to future work.

\section{Limiting Aggregation to Max}\label{sec:max}

In this section we study the expressivity of \emph{monotonic max GNNs}, which
follow the same restrictions as monotonic max-sum GNNs but additionally allow
only for the max aggregation function. Theorem~\ref{thm:maxGNN-rules-complete}
shows that each such GNN corresponds to a Datalog program without inequalities.
Consequently, monotonic max GNNs cannot count the connections of a constant in
a dataset.

\begin{definition}
    A \emph{monotonic max} $(\col,\ddim)$\emph{-GNN} is a monotonic max-sum GNN
    that uses the $\maxsum{1}$ aggregation function in all layers.
\end{definition}

\begin{restatable}{theorem}{equivalencemax}\label{thm:maxGNN-rules-complete}
    For each monotonic max $(\col,\ddim)$-GNN $\GNN$ with $L$ layers, let
    ${\ddim[\GNN] = \max(\ddim[0], \dots, \ddim[L])}$, and let $\Prog_\GNN$ be
    the Datalog program containing up to variable renaming each $(L,|\col|
    \cdot \ddim[\GNN])$-tree-like rule without inequalities captured by $\GNN$.
    Then, $\GNN$ and $\Prog_\GNN$ are equivalent.
\end{restatable}

\citea{tccgkm22explainable-gnn-models} presented a closely related
characterisation for MGNNs, and the main difference is that we use the
canonical encoding. The latter allows us to describe the target Datalog class
more precisely, which in turn allows us to prove the converse: each Datalog
program with only tree-like rules and without inequalities is equivalent to a
monotonic max GNN.

In what follows, we fix a program $\Prog$ consisting of $(d,f)$-tree-like rules
without inequalities. Recall that the signature of $\Prog$ consists of unary
predicates ${U_1, \dots, U_{\ddim}}$ and binary predicates $\edg{c}$ for ${c
\in \col}$. Now let ${\tau_1, \dots, \tau_n}$ be a sequence containing up to
variable renaming each $(d,f)$-tree-like formula for variable $x$ without
inequalities ordered by increasing depth---that is, for all ${i < j}$, the
depth of $\tau_i$ is less than or equal to the depth of $\tau_j$. Each $\tau_i$
can be written as
\begin{align}
    \tau_i = \varphi_{i,0} \wedge \bigwedge_{k=1}^{m_i} \Big( \edg{c_k}(x,y_k) \wedge \varphi_{i,k} \Big), \label{eq:tau:i}
\end{align}
where $\varphi_{i,0}$ is a conjunction of unary atoms using only variable $x$,
each $\varphi_{i,k}$ with $1 \leq k \leq m_i$ is a $(d-1,f)$-tree-like formula
for $y_k$, and, for all ${1 \leq k < k' \leq m_i}$, formulas $\varphi_{i,k}$
and $\varphi_{i,k'}$ do not have variables in common. Note that formulas
$\varphi_{i,k}$ can be $\top$, and that colours $c_k$ need not be distinct.

We define $\GNN_\Prog$ as the monotonic max $(\col,\ddim)$-GNN of form
\eqref{eq:GNN} satisfying the following conditions. The number of layers is ${L
= d+2}$, the activation function is ReLU, and the classification function
$\cls$ is the step function with threshold $1$. For ${1 \leq \ell < L}$,
dimension $\ddim[\ell]$ is defined as the number of formulas in the above
sequence of depth at most ${\ell - 1}$. The elements of $\matA{\ell}$,
$\matB{\ell}{c}$, and $\bias{\ell}$ are defined as follows, for ${c \in \col}$,
${1 \leq \ell \leq L}$, ${1 \leq i \leq \ddim[\ell]}$, and ${1 \leq j \leq
\ddim[\ell-1]}$.
\begin{align*}
    (\matA{\ell})_{i,j} = \left\{
    \begin{array}{@{\;\;}l@{\;\;}l@{\;}l@{}}
        1   & \text{if} & \\
            & \tabitem  & \ell = 1 \text{ and } \tau_i \text{ contains } U_j(x); \text{ or} \\
            & \tabitem  & 2 \leq \ell < L \text{ and} \\
            &           & - \; 1 \leq i \leq \ddim[\ell-1] \text{ and } i = j, \text{ or } \\
            &           & - \; \ddim[\ell-1] < i \leq \ddim[\ell] \text{ and } \varphi_{i,0} = \tau_j; \text{ or} \\
            & \tabitem  & \ell = L \text{ and } \Prog \text{ contains rule} \\
            &           & \tau_j \rightarrow U_i(x) \text{ up to variable renaming;} \\[1ex]
        0   & \multicolumn{2}{@{}l@{}}{\text{otherwise.}} \\
    \end{array}
    \right. \\[2ex]
    \elt{\matB{\ell}{c}}{i,j} = \left\{
    \begin{array}{@{\;\;}l@{\;\;}l@{}}
        1   & \text{if } 2 \leq \ell < L \text{ and there exists } 1 \leq k \leq m_i \\
            & \text{such that } c = c_k \text{ and } \varphi_{i,k} \text{ and } \tau_j \\
            & \text{are equal up to variable renaming;} \\[1ex]
        0   & \text{otherwise.} \\
    \end{array}
    \right. \\[2ex]
    \elt{\bias{\ell}}{i} = \left\{
    \begin{array}{@{}l@{\;}l@{}}
        1 - \sum\limits_{j=1}^{\ddim[\ell-1]} (\elt{\matA{\ell}}{i,j} + \! \sum\limits_{c \in \col} \elt{\matB{\ell}{c}}{i,j})  &
        \begin{array}{@{}l@{}}
            \text{if } \ell = 1, \text{ or } \\
            1 \leq \ell < L \text{ and } \\
            \ddim[\ell-1] < i \leq \ddim[\ell]; \\
        \end{array}
        \\[4ex]
        0   & \text{otherwise.} \\
    \end{array}
    \right.
\end{align*}

To understand the intuition behind the construction of $\GNN_\Prog$, assume
that $\GNN_\Prog$ is applied to a dataset $D$, and consider a vector
$\vlab{v}{\ell}$ labelling in layer $\ell$ a vertex corresponding to some term
$t$ of $D$. Then, the $i$-th component of $\vlab{v}{\ell}$ is paired with
formula $\tau_i$ from the above enumeration, and it indicates whether it is
possible to evaluate $\tau_i$ over $D$ by mapping variable $x$ to $t$. This is
formally captured by Lemma~\ref{lem:max:equivalence}. To ensure that
$\GNN_\Prog$ and $\Prog$ are equivalent, layer $L$ of $\GNN_\Prog$ simply
realises a disjunction over all rules in the program.

\begin{restatable}{lemma}{maxequivalence}\label{lem:max:equivalence}
    For each $(\col,\ddim)$-dataset $D$, layer ${1 \leq \ell < L}$ of
    $\GNN_\Prog$, position ${1 \leq i \leq \ddim[\ell]}$, and term $t$ in $D$,
    and for $\vlab{v}{\ell}$ the labelling of the vertex corresponding to $t$
    when $\GNN_\Prog$ is applied to the canonical encoding of $D$,
    \begin{itemize}
        \item ${\elt{\vlab{v}{\ell}}{i} = 1}$ if there exists a substitution
        $\nu$ mapping $x$ to $t$ such that ${D \models \tau_i\nu}$, and

        \item ${\elt{\vlab{v}{\ell}}{i} = 0}$ otherwise.
    \end{itemize}
\end{restatable}

Note that each $\ddim[\ell]$ with ${1 \leq \ell < L}$ is determined by the
number of $(d,f)$-tree-like formulas of depth $\ell - 1$, and that $\ddim[L-1]$
is the largest such number. We next determine an upper bound on $\ddim[L-1]$.
By Definition~\ref{def:tree-like}, the fan-out of a variable of depth $i$ is at
most $f(d-i)$. The number of variables of depth $i$ is at most the number of
variables of depth $i-1$ times the fan-out of each variable, which is ${f^i
\cdot d \dots (d - i + 1)}$ and is bounded by ${f^i \cdot d!}$. By adding up
the contribution for each depth, there are at most ${f^d \cdot (d+1)!}$
variables. Each variable is labelled by one of the ${2^{\ddim}}$ conjunctions
of depth zero, and each non-root variable is connected by one of the $|\col|$
predicates to its parent. Hence, there are at most ${(|\col| \cdot
2^{\ddim})^{f^d \cdot (d+1)!}}$ tree-like formulas.

\begin{restatable}{theorem}{progequivalencemax}\label{thm:prog:equivalence:max}
    Program $\Prog$ and GNN $\GNN_\Prog$ are equivalent, and moreover
    ${\ddim[L-1] \leq (|\col| \cdot 2^{\ddim})^{f^d \cdot (d+1)!}}$.
\end{restatable}

\section{Conclusion}

We have shown that each monotonic max-sum GNN (i.e., a GNN that uses max and
sum aggregation functions and satisfies certain properties) is equivalent to a
Datalog program with inequalities in the sense that applying the GNN or a
single round of the rules of the program to any dataset produces the same
result. We have also sharpened this result to monotonic max GNNs and shown the
converse: each tree-like Datalog program without inequalities is equivalent to
a monotonic max GNN. We see many avenues for future work. First, we aim to
completely characterise monotonic max-sum GNNs. Second, we intend to implement
rule extraction. Third, we shall investigate the empirical performance of
monotonic max-sum GNNs on tasks other than link prediction, such as node
classification.

\section*{Acknowledgements}

This work was supported by the SIRIUS Centre for Scalable Data Access (Research
Council of Norway, project number 237889), and the EPSRC projects ConCur
(EP/V050869/1), UK FIRES (EP/S019111/1), and AnaLOG (EP/P025943/1). For the
purpose of Open Access, the author has applied a CC BY public copyright licence
to any Author Accepted Manuscript (AAM) version arising from this submission.

\bibliographystyle{kr}
\bibliography{references}

\begin{thebibliography}{}

\bibitem[\protect\citeauthoryear{Abiteboul, Hull, and
  Vianu}{1995}]{abiteboul95foundation}
Abiteboul, S.; Hull, R.; and Vianu, V.
\newblock 1995.
\newblock {\em {Foundations of Databases}}.
\newblock Addison Wesley.

\bibitem[\protect\citeauthoryear{Baader \bgroup et al\mbox.\egroup
  }{2007}]{dl-handbook-2}
Baader, F.; Calvanese, D.; McGuinness, D.; Nardi, D.; and Patel-Schneider,
  P.~F., eds.
\newblock 2007.
\newblock {\em {The Description Logic Handbook: Theory, Implementation and
  Applications}}.
\newblock Cambridge University Press, 2nd edition.

\bibitem[\protect\citeauthoryear{Bader \bgroup et al\mbox.\egroup
  }{2007}]{DBLP:conf/ijcai/BaderHHW07}
Bader, S.; Hitzler, P.; H{\"{o}}lldobler, S.; and Witzel, A.
\newblock 2007.
\newblock {A Fully Connectionist Model Generator for Covered First-Order Logic
  Programs}.
\newblock In {\em Proc.\ IJCAI},  666--671.

\bibitem[\protect\citeauthoryear{Bader, d'Avila Garcez, and
  Hitzler}{2005}]{DBLP:conf/flairs/BaderGH05}
Bader, S.; d'Avila Garcez, A.~S.; and Hitzler, P.
\newblock 2005.
\newblock {Computing First-Order Logic Programs by Fibring Artificial Neural
  Networks}.
\newblock In {\em Proc.\ FLAIRS},  314--319.
\newblock {AAAI} Press.

\bibitem[\protect\citeauthoryear{Barcel{\'{o}} \bgroup et al\mbox.\egroup
  }{2020}]{DBLP:conf/iclr/BarceloKM0RS20}
Barcel{\'{o}}, P.; Kostylev, E.~V.; Monet, M.; P{\'{e}}rez, J.; Reutter, J.~L.;
  and Silva, J.~P.
\newblock 2020.
\newblock {The Logical Expressiveness of Graph Neural Networks}.
\newblock In {\em Proc.\ ICLR}.

\bibitem[\protect\citeauthoryear{Campero \bgroup et al\mbox.\egroup
  }{2018}]{DBLP:journals/corr/abs-1809-02193}
Campero, A.; Pareja, A.; Klinger, T.; Tenenbaum, J.; and Riedel, S.
\newblock 2018.
\newblock Logical rule induction and theory learning using neural theorem
  proving.
\newblock {\em CoRR} abs/1809.02193.

\bibitem[\protect\citeauthoryear{Dantsin \bgroup et al\mbox.\egroup
  }{2001}]{DBLP:journals/csur/DantsinEGV01}
Dantsin, E.; Eiter, T.; Gottlob, G.; and Voronkov, A.
\newblock 2001.
\newblock Complexity and expressive power of logic programming.
\newblock {\em {ACM} Comput. Surv.} 33(3):374--425.

\bibitem[\protect\citeauthoryear{Dong \bgroup et al\mbox.\egroup
  }{2019}]{DBLP:conf/iclr/DongMLWLZ19}
Dong, H.; Mao, J.; Lin, T.; Wang, C.; Li, L.; and Zhou, D.
\newblock 2019.
\newblock {Neural Logic Machines (Poster)}.
\newblock In {\em Proc.\ ICLR}.

\bibitem[\protect\citeauthoryear{H{\"{o}}lldobler, Kalinke, and
  St{\"{o}}rr}{1999}]{DBLP:journals/apin/HolldoblerKS99}
H{\"{o}}lldobler, S.; Kalinke, Y.; and St{\"{o}}rr, H.-P.
\newblock 1999.
\newblock {Approximating the Semantics of Logic Programs by Recurrent Neural
  Networks}.
\newblock {\em Applied Intelligence} 11(1):45--58.

\bibitem[\protect\citeauthoryear{Huang \bgroup et al\mbox.\egroup
  }{2023}]{DBLP:journals/corr/abs-2302-02209}
Huang, X.; Orth, M. A.~R.; Ceylan, {\.I}.~{\.I}.; and Barcel{\'{o}}, P.
\newblock 2023.
\newblock A theory of link prediction via relational weisfeiler-leman.
\newblock {\em CoRR} abs/2302.02209.

\bibitem[\protect\citeauthoryear{Ioannidis, Marques, and
  Giannakis}{2019}]{DBLP:conf/icassp/IoannidisMG19}
Ioannidis, V.~N.; Marques, A.~G.; and Giannakis, G.~B.
\newblock 2019.
\newblock A recurrent graph neural network for multi-relational data.
\newblock In {\em Proc.\ ICASSP},  8157--8161.
\newblock {IEEE}.

\bibitem[\protect\citeauthoryear{Kipf and
  Welling}{2017}]{DBLP:conf/iclr/KipfW17}
Kipf, T.~N., and Welling, M.
\newblock 2017.
\newblock Semi-supervised classification with graph convolutional networks.
\newblock In {\em Proc.\ ICLR}.

\bibitem[\protect\citeauthoryear{Liu \bgroup et al\mbox.\egroup
  }{2021}]{DBLP:conf/nips/LiuGHK21}
Liu, S.; {Cuenca Grau}, B.; Horrocks, I.; and Kostylev, E.~V.
\newblock 2021.
\newblock {INDIGO:} gnn-based inductive knowledge graph completion using
  pair-wise encoding.
\newblock In {\em Proc.\ NeurIPS},  2034--2045.

\bibitem[\protect\citeauthoryear{Morris \bgroup et al\mbox.\egroup
  }{2019}]{DBLP:conf/aaai/0001RFHLRG19}
Morris, C.; Ritzert, M.; Fey, M.; Hamilton, W.~L.; Lenssen, J.~E.; Rattan, G.;
  and Grohe, M.
\newblock 2019.
\newblock {Weisfeiler and Leman Go Neural: Higher-Order Graph Neural Networks}.
\newblock In {\em Proc.\ AAAI},  4602--4609.
\newblock {AAAI} Press.

\bibitem[\protect\citeauthoryear{Motik \bgroup et al\mbox.\egroup
  }{2012}]{owl2profiles}
Motik, B.; {Cuenca Grau}, B.; Horrocks, I.; Wu, Z.; Fokoue, A.; and Lutz, C.
\newblock 2012.
\newblock {\em OWL 2 Web Ontology Language: Profiles (2nd Edition)}.
\newblock World Wide Web Consortium.

\bibitem[\protect\citeauthoryear{P{\'{e}}rez, Arenas, and
  Gutierrez}{2009}]{DBLP:journals/tods/PerezAG09}
P{\'{e}}rez, J.; Arenas, M.; and Gutierrez, C.
\newblock 2009.
\newblock Semantics and complexity of {SPARQL}.
\newblock {\em {ACM} Trans. Database Syst.} 34(3):16:1--16:45.

\bibitem[\protect\citeauthoryear{Pflueger, {Tena Cucala}, and
  Kostylev}{2022}]{DBLP:conf/semweb/PfluegerCK22}
Pflueger, M.; {Tena Cucala}, D.~J.; and Kostylev, E.~V.
\newblock 2022.
\newblock {GNNQ:} {A} neuro-symbolic approach to query answering over
  incomplete knowledge graphs.
\newblock In {\em Proc.\ ISWC}, volume 13489 of {\em Lecture Notes in Computer
  Science},  481--497.
\newblock Springer.

\bibitem[\protect\citeauthoryear{Qu, Bengio, and
  Tang}{2019}]{DBLP:conf/icml/QuBT19}
Qu, M.; Bengio, Y.; and Tang, J.
\newblock 2019.
\newblock {GMNN:} graph markov neural networks.
\newblock In {\em Proc.\ ICML}, volume~97 of {\em Proceedings of Machine
  Learning Research},  5241--5250.

\bibitem[\protect\citeauthoryear{Rockt{\"{a}}schel and
  Riedel}{2017}]{DBLP:conf/nips/Rocktaschel017}
Rockt{\"{a}}schel, T., and Riedel, S.
\newblock 2017.
\newblock {End-to-end Differentiable Proving}.
\newblock In {\em Proc.\ NeurIPS},  3788--3800.

\bibitem[\protect\citeauthoryear{Sadeghian \bgroup et al\mbox.\egroup
  }{2019}]{DBLP:conf/nips/SadeghianADW19}
Sadeghian, A.; Armandpour, M.; Ding, P.; and Wang, D.~Z.
\newblock 2019.
\newblock {DRUM: End-To-End Differentiable Rule Mining On Knowledge Graphs}.
\newblock In {\em Proc.\ NeurIPS},  15321--15331.

\bibitem[\protect\citeauthoryear{Schlichtkrull \bgroup et al\mbox.\egroup
  }{2018}]{DBLP:conf/esws/SchlichtkrullKB18}
Schlichtkrull, M.~S.; Kipf, T.~N.; Bloem, P.; van~den Berg, R.; Titov, I.; and
  Welling, M.
\newblock 2018.
\newblock Modeling relational data with graph convolutional networks.
\newblock In {\em Proc.\ ESWC}, volume 10843 of {\em LNCS},  593--607.
\newblock Springer.

\bibitem[\protect\citeauthoryear{Sourek, Zelezn{\'{y}}, and
  Kuzelka}{2021}]{DBLP:journals/ml/SourekZK21}
Sourek, G.; Zelezn{\'{y}}, F.; and Kuzelka, O.
\newblock 2021.
\newblock Beyond graph neural networks with lifted relational neural networks.
\newblock {\em Mach. Learn.} 110(7):1695--1738.

\bibitem[\protect\citeauthoryear{{Tena Cucala} \bgroup et al\mbox.\egroup
  }{2022}]{tccgkm22explainable-gnn-models}
{Tena Cucala}, D.; {Cuenca Grau}, B.; Kostylev, E.~V.; and Motik, B.
\newblock 2022.
\newblock {Explainable GNN-Based Models over Knowledge Graphs}.
\newblock In {\em Proc.\ ICLR}.

\bibitem[\protect\citeauthoryear{{Tena Cucala}, {Cuenca Grau}, and
  Motik}{2022}]{tccgm22faithful-approaches}
{Tena Cucala}, D.; {Cuenca Grau}, B.; and Motik, B.
\newblock 2022.
\newblock {Faithful Approaches to Rule Learning}.
\newblock In {\em Proc.\ KR},  484--493.

\bibitem[\protect\citeauthoryear{Teru, Denis, and
  Hamilton}{2020}]{DBLP:conf/icml/TeruDH20}
Teru, K.~K.; Denis, E.~G.; and Hamilton, W.~L.
\newblock 2020.
\newblock Inductive relation prediction by subgraph reasoning.
\newblock In {\em Proc.\ ICML}, volume 119 of {\em Proceedings of Machine
  Learning Research},  9448--9457.
\newblock {PMLR}.

\bibitem[\protect\citeauthoryear{Yang, Cohen, and
  Salakhutdinov}{2016}]{DBLP:conf/icml/YangCS16}
Yang, Z.; Cohen, W.~W.; and Salakhutdinov, R.
\newblock 2016.
\newblock Revisiting semi-supervised learning with graph embeddings.
\newblock In {\em Proc.\ ICML}, volume~48 of {\em {JMLR} Workshop and
  Conference Proceedings},  40--48.

\bibitem[\protect\citeauthoryear{Yang, Yang, and
  Cohen}{2017}]{DBLP:conf/nips/YangYC17}
Yang, F.; Yang, Z.; and Cohen, W.~W.
\newblock 2017.
\newblock {Differentiable Learning of Logical Rules for Knowledge Base
  Reasoning}.
\newblock In {\em Proc.\ NeurIPS},  2319--2328.

\bibitem[\protect\citeauthoryear{Zhang and
  Chen}{2018}]{DBLP:conf/nips/ZhangC18}
Zhang, M., and Chen, Y.
\newblock 2018.
\newblock Link prediction based on graph neural networks.
\newblock In {\em Proc.\ NeurIPS},  5171--5181.

\end{thebibliography}

\iftoggle{withappendix}{
    \newpage
    \onecolumn
    \appendix

    \counterwithin{theorem}{section}
    \renewcommand{\theproposition}{\thesection.\arabic{theorem}}
    \renewcommand{\thecorollary}{\thesection.\arabic{theorem}}
    \renewcommand{\thelemma}{\thesection.\arabic{theorem}}
    \renewcommand{\thedefinition}{\thesection.\arabic{theorem}}
    \renewcommand{\theclaim}{\thesection.\arabic{claim}}
    \renewcommand{\theexample}{\thesection.\arabic{theorem}}

    \section{Proofs for Section~\ref{sec:maxsum}}\label{sex:proofs-maxsum}

Throughout this appendix, we fix a max-sum GNN $\GNN$, dimensions ${\ddim[0],
\dots, \ddim[L]}$, and aggregation functions ${\aggK{1}, \dots, \aggK{L}}$ as
specified in Section~\ref{sec:maxsum:limiting}. As in
Section~\ref{sec:maxsum:enumerating}, for $\vc{x}$ a vector, $j$ an index, and
$v$ a real number, $\vc{x}[j \gets v]$ is the vector obtained from $\vc{x}$ by
replacing its $j$th component with $v$.

To prove our results, we shall define a nonempty sequence $\Se{\ell}{i}$ for
each ${0 \leq \ell \leq L}$ and ${1 \leq i \leq \ddim[\ell]}$; intuitively,
each $\Se{\ell}{i}$ enumerates $\X{\ell}{i}$ in ascending order. Our definition
is inductive and uses two auxiliary notions that we define next. In particular,
consider an arbitrary $\ell$ with ${0 < \ell \leq L}$ and assume that
$\Se{\ell-1}{i}$ have been defined for all ${1 \leq i \leq \ddim[\ell]}$. Then,
$\s{\ell-1}$ is the vector of dimension $\ddim[\ell-1]$ such that
$\elt{\s{\ell-1}}{i}$ is the first element of $\Se{\ell}{i}$ for each ${1 \leq
i \leq \ddim[\ell-1]}$. Moreover, a \emph{$(\ell,i)$-triple} is a triple of the
form $\triple{\x}{\Ys}{z}$ whose components satisfy the following conditions:
\begin{itemize}
    \item $\x$ is a vector of dimension $\ddim[\ell-1]$ such that ${\elt{\x}{j}
    \in \Se{\ell-1}{j}}$ holds for all ${1 \leq j \leq \ddim[\ell-1]}$;

    \item $\Ys$ is a $(\col,\ell-1)$-multiset family such that ${\elt{\y}{j}
    \in \Se{\ell-1}{j}}$ holds for all ${c \in \col}$, ${\y \in \Ys[c]}$, and
    ${1 \leq j \leq \ddim[\ell-1]}$; and

    \item ${z = \Val{\ell}{i}{\x}{\Ys}}$.
\end{itemize}
An $(\ell,i)$-triple $\triple{\x_2}{\Ys_2}{z_2}$ is a successor of an
$(\ell,i)$-triple $\triple{\x_1}{\Ys_1}{z_1}$ if exactly one of the following
conditions holds:
\begin{itemize}
    \item ${\Ys_1 = \Ys_2}$ and ${\x_2 = \x_1[j \gets x']}$ for some ${1 \leq j
    \leq \ddim[\ell-1]}$ and $x'$ the element that succeeds $\elt{\x}{j}$ in
    $\Se{\ell-1}{j}$; or

    \item ${\x_2 = \x_1}$ and there exist a colour ${c \in \col}$, vector ${\y
    \in \Ys[c]_1}$, and index ${1 \leq j \leq \ddim[\ell-1]}$ such that
    ${\Ys[c']_2 = \Ys[c']_1}$ for each colour ${c' \in \col \setminus \{ c
    \}}$, and ${\Ys[c]_2 = (\Ys[c]_1 \setminus \{ \y \}) \cup \{ \y[j \gets y']
    \}}$ where $y'$ is the element that succeeds $\elt{\y}{j}$ in
    $\Se{\ell-1}{j}$; or

    \item ${\x_2 = \x_1}$ and there exist a colour ${c \in \col}$ and index ${1
    \leq j \leq \ddim[\ell-1]}$ such that ${\Ys[c']_2 = \Ys[c']_1}$ for each
    colur ${c' \in \col \setminus \{ c \}}$, and ${\Ys[c]_2 = \Ys[c]_1 \cup \{
    \s{\ell-1}[j \gets y'] \}}$ where $y'$ is the first positive element of
    $\Se{\ell-1}{j}$.
\end{itemize}

We are now ready to define sequences $\Se{\ell}{i}$ for all ${0 \leq \ell \leq
L}$ and ${1 \leq i \leq \ddim[\ell]}$.
\begin{itemize}
    \item For the base case ${\ell = 0}$, let ${\Se{0}{i} = (0,1)}$ for each
    ${1 \leq i \leq \ddim[0]}$.

    \item For the inductive step, assume that $\Se{\ell-1}{i}$ has been defined
    for each ${1 \leq i \leq \ddim[\ell-1]}$, and consider arbitrary ${1 \leq i
    \leq \ddim[\ell]}$. To define $\Se{\ell}{i}$, we first define an auxiliary
    sequence $\F{\ell}{i}$ of finite sets of $(\ell,i)$-triples as follows.
    \begin{itemize}
        \item For the base case, the first element of $\F{\ell}{i}$ is ${f_0 =
        \{ \triple{\s{\ell-1}}{\Ys_\emptyset}{z_0} \}}$, where $\Ys_\emptyset$
        is such that ${\Ys[c]_\emptyset = \emptyset}$ for each ${c \in \col}$
        and ${z_0 = \act(\Val{\ell}{i}{\s{\ell-1}}{\Ys_{\emptyset}})}$.

        \item For the inductive step, assuming that $f_{n-1}$ with $n > 0$ has
        been defined and is not empty, let
        \begin{displaymath}
        \begin{array}{@{}l@{\;\;}l@{}}
            f_n =   & \{ \triple{\x}{\Ys}{z} \in f_{n-1}\mid z > \min(f_{n-1}) \} \; \cup \\[1ex]
                    & \{ \triple{\x}{\Ys}{z} \mid z > \min(f_{n-1}) \text{ and } \triple{\x}{\Ys}{z} \mbox{ is a successor of some } \triple{\x'}{\Yps}{\min(f_{n-1})} \in f_{n-1} \}, \\
        \end{array}
        \end{displaymath}
        where $\min(f_n)$ is the minimum number appearing in the third position
        of an $(\ell,i)$-triple in $f_{n-1}$; such a number always exists since
        $f_{n-1}$ is never empty and it contains a finite number of triples.
        Then, $\Se{\ell}{i}$ is the sequence of real numbers whose $n$-th
        element is $\act(\min(f_n))$ if $f_n$ is defined and ${f_n \neq
        \emptyset}$. Since $f_0$ is always defined and not empty,
        $\Se{\ell}{i}$ is not empty.
    \end{itemize}
\end{itemize}

The following lemma shows that sequences $\Se{\ell}{i}$ capture our intuition
mentioned above.

\begin{lemma}\label{lem:S}
    For each ${1 \leq \ell \leq L}$ and each ${1 \leq i \leq \ddim[\ell]}$,
    sequence $\Se{\ell}{i}$ satisfies the following conditions:
    \begin{enumerate}[start=1,label={(S\arabic*)},leftmargin=30pt]
        \item\label{cond:nonnegative}
        each element of $\Se{\ell}{i}$ is nonnegative;

        \item\label{cond:monotonic}
        $\Se{\ell}{i}$ is strictly monotonically increasing;

        \item\label{cond:finite-diverges}
        $\Se{\ell}{i}$ is either finite or it converges to $\infty$; and

        \item\label{cond:s-x}
         the set of elements in $\Se{\ell}{i}$ is $\X{\ell}{i}$.
    \end{enumerate}
\end{lemma}

\begin{proof}
We prove all four conditions by induction over $\ell$. For the base case
$\ell=0$, sequence $\Se{0}{i}$ is by definition is equal to $(0,1)$ for each
${1 \leq i \leq \ddim[\ell]}$, so conditions
\ref{cond:nonnegative}--\ref{cond:s-x} hold trivially. Now consider arbitrary
${1 \leq \ell \leq L}$ such that each $\Se{\ell-1}{j}$ with ${1 \leq j \leq
\ddim[\ell-1]}$ satisfies conditions \ref{cond:nonnegative}--\ref{cond:s-x},
and consider arbitrary ${1 \leq i \leq \ddim[\ell]}$.

Condition \ref{cond:nonnegative} follows straightforwardly from the fact that
each element of $\Se{\ell}{i}$ for ${\ell \geq 1}$ is the image of $\act$ for
some real number $z$, and $\act(z) \geq 0$ for all $z \in \real$, since the
range of $\act$ is $\nnreal$. Condition \ref{cond:monotonic} follows from the
fact that, for each ${n \in \mathbb{N}}$, each triple ${\triple{\x}{\Ys}{z} \in
f_n \setminus f_{n-1}}$ satisfies $z > \min(f_{n-1})$, and so ${\min(f_n) >
\min(f_{n-1})}$ holds.

\medskip

We prove Condition \ref{cond:finite-diverges} by contradiction---that is, we
assume that $\Se{\ell}{i}$ is infinite, and that there exists some
$\bar{\afteract{}} \in \real$ such that each element of $\Se{\ell}{i}$ is
smaller than $\bar{\afteract{}}$.

Consider an arbitrary element $\afteract{}$ in $\Se{\ell}{i}$. By the
definition of $\Se{\ell}{i}$, there exists a triple $\triple{\x}{\Ys}{z}$ such
that ${z = \Val{\ell}{i}{\x}{\Ys}}$, $\act(z) = \afteract{}$, and, for each ${c
\in \col}$ and ${\y \in \Ys[c]}$, vector $\y$ contains a nonzero element. Let
$\bar{\beforeact{}}$ be the smallest natural number such that
${\act(\bar{\beforeact{}}) \geq \bar{\afteract{}}}$; such $\bar{\beforeact{}}$
exists since $\act$ is unbounded. For each ${1 \leq j \leq \ddim[\ell-1]}$ and
each ${c \in \col}$, let $\minW{j}$, $\afteract{j}$, and $n_{j,c}$ be as
follows:
\begin{align}
    \minW{j} =      & \min \{ \elt{\matA{\ell}}{i,j}\} \cup \{ \elt{\matB{\ell}{c}}{i,j} \mid c \in \col \}; \\[1ex]
    \afteract{j} =  & \begin{cases}
                        \frac{\bar{\beforeact{}} - \elt{\bias{\ell}}{i}}{\minW{\ell}}   & \text{if } \minW{\ell} \neq 0, \\
                        \text{undefined}                                                & \text{otherwise}; \\
                      \end{cases} \\[1ex]
    \minX{j} =      & \begin{cases}
                        \text{the first positive value of } \Se{\ell-1}{j}  & \text{if such a value exists}, \\
                        \text{undefined}                                    & \text{otherwise}; \\
                      \end{cases} \\[1ex]
    n_{j,c} =       & \begin{cases}
                        \lceil \frac{\bar{\beforeact{}} - \elt{\bias{\ell}}{i}}{\elt{\matB{\ell}{c}}{i,j} \cdot \minX{j}} \rceil    & \text{if } \elt{\matB{\ell}{c}}{i,j} \neq 0 \text{ and } \minX{j} \text{ is defined}; \\
                        0                                                                                                           & \text{otherwise}. \\
                      \end{cases}
\end{align}
We next show the following properties:
\begin{enumerate}
    \item if ${\elt{\matA{\ell}}{i,j} > 0}$, then ${\elt{\x}{j} <
    \afteract{j}}$;

    \item for each ${c \in \col}$, if ${\elt{\matB{\ell}{c}}{i,j} > 0}$, then
    ${\elt{\y}{j} < \afteract{j}}$ for each ${\y \in \Ys[c]}$;

    \item for each ${c \in \col}$, if ${\elt{\matB{\ell}{c}}{i,j} > 0}$, then
    there exist fewer than $n_{j,c}$ elements in $\Ys[c]$ whose $j$-th element
    is not zero.
\end{enumerate}
To see the first property, note that if ${\elt{\x}{j} \geq \afteract{j}}$, then
condition \ref{cond:nonnegative} of the inductive hypothesis ensures that all
elements of $\x$ and vectors in $\Ys$ are nonnegative; since the weights of
$\matA{\ell}$ and $\matB{\ell}{c}$ are also nonnegative, we have
\begin{displaymath}
    \afteract{} = \act\Big(\Val{\ell}{i}{\x}{\Ys}\Big) \geq \act\Big( \minW{j} \afteract{j} + \elt{\bias{\ell}}{i} \Big) = \act(\bar{\beforeact{}}) \geq \bar{\afteract{}},
\end{displaymath}
which contradicts our assumption that all elements of $\Se{\ell}{i}$ are
smaller than $\bar{\afteract{}}$. The second property follows analogously. To
see the third property, assume that there exist at least $n_{j,c}$ vectors $\y$
in $\Ys[c]$ such that $\elt{\y}{j} > 0$. Then,
\begin{displaymath}
    \afteract{} = \act\Big(\Val{\ell}{i}{\x}{\Ys}\Big) \geq \act\Big( \elt{\matB{\ell}{c}}{i,j} \sum_{\y \in \Ys[c]} \elt{\y}{j} + \elt{\bias{\ell}}{i} \Big) \geq \act\Big( \elt{\matB{\ell}{c}}{i,j} \cdot n_j \cdot \minX{j} + \elt{\bias{\ell}}{i} \Big) \geq \act( \bar{\beforeact{}}) \geq \bar{\afteract{}},
\end{displaymath}
which again contradicts our assumption that all elements of $\Se{\ell}{i}$ are
smaller than $\bar{\afteract{}}$.

By conditions \ref{cond:monotonic} and \ref{cond:finite-diverges} of the
inductive hypothesis, each $\Se{\ell-1}{j}$ is countable, monotonically
increasing, and either finite or converges to infinity; hence, set ${\{ s \in
\Se{\ell-1}{j} \mid s \leq \bar{\afteract{}} \}}$ is finite. Thus, by the three
properties shown above, if $\elt{\matA{\ell}}{i,j} > 0$, then $\elt{\x}{j}$ can
only take finitely many values; similarly, for all ${c \in \col}$ and ${\y \in
\Ys[c]}$, if ${\elt{\matB{\ell}{c}}{i,j} > 0}$, then, $\elt{\y}{j}$ can only
take finitely many values. Notice also that each $\Ys[c]$ cannot have
infinitely many elements, since it does not contain any vector where all
elements are $0$, and, for each ${1 \leq j \leq \ddim[\ell-1]}$, there exist
fewer than $n_{j,c}$ elements in $\Ys[c]$ whose $j$-th component's value is
positive, and none for which it is negative. Hence, there are only finitely
many values that $\act(z)$ can take. Thus, $\Se{\ell}{i}$ is finite, which
contradicts our initial assumption.

\medskip

Finally, we show condition \ref{cond:s-x}. To this end, we first show that all
elements of $\Se{\ell}{i}$ are in $\X{\ell}{i}$. Consider an arbitrary element
$\alpha \in \Se{\ell}{i}$; by definition, there exists an $(\ell,i)$-triple of
the form $\triple{\x}{\Ys}{z}$ such that $\act(z) = \alpha$. Now, for all ${1
\leq j \leq \ddim[\ell-1]}$, ${c \in \col}$, and ${\y \in \Ys[c]}$, the
definition of an $(\ell,i)$-triple ensures ${\elt{\x}{j} \in \Se{\ell-1}{j}}$
and ${\elt{\y}{j} \in \Se{\ell-1}{j}}$; thus, our inductive hypothesis implies
${\elt{\x}{j} \in \X{\ell-1}{j}}$ and ${\elt{\y}{j} \in \X{\ell-1}{j}}$. But
then, the definition of an $(\ell,i)$-triple ensures ${z =
\Val{\ell}{i}{\x}{\Ys}}$, and the definition of $\X{\ell}{i}$ ensures ${\act(z)
\in \X{\ell}{i}}$, as required.

To prove that each element of $\X{\ell}{i}$ appears in $\Se{\ell}{i}$, consider
arbitrary ${\alpha \in \X{\ell}{i}}$. By Definition~\ref{def:possible-values},
there exists a vector $\x_\alpha$ of dimension $\ddim[\ell-1]$ where
$\elt{\x_\alpha}{j} \in \X{\ell-1}{j}$ for each ${1 \leq j \leq
\ddim[\ell-1]}$, and there also exists a $(\col,\ell-1)$-multiset family
$\Ys_\alpha$ such that ${\elt{\y}{j} \in \X{\ell-1}{j}}$ holds for each $c \in
\col$, each ${\y \in \Ys[c]_\alpha}$, and each ${1 \leq j \leq \ddim[\ell-1]}$;
moreover, ${\act(z_\alpha) = \alpha}$ for ${z_\alpha =
\Val{\ell}{i}{\x_\alpha}{\Ys_\alpha}}$. By induction hypothesis, all elements
in $\X{\ell-1}{j}$ are in $\Se{\ell-1}{j}$. Hence, there exists at least one
finite sequence
${\triple{\s{\ell-1}}{\Ys_\emptyset}{\Val{\ell}{i}{\s{\ell-1}}{\Ys_\emptyset}}
= t_0, \dots, t_K = \triple{\x_\alpha}{\Ys_\alpha}{z_\alpha}}$ such that $t_n$
is a successor of $t_{n-1}$ for each ${1 \leq n \leq K}$. Indeed, each multiset
$\Ys[c]_\alpha$ is finite and, starting from $t_0$, we can reach $t_K$ by, in
each step, changing some vector component to the next element in
$\Se{\ell-1}{j}$ or adding a new vector to some multiset of the
$(\col,\ell-1)$-multiset family. We now show by induction over ${0 \leq n \leq
K}$ the following statement ($\ast$): for each ${t_n =
\triple{\x_n}{\Ys_n}{z_n}}$ in this sequence, some element of $\F{\ell}{i}$
contains a $(\ell,i)$-triple $\triple{\x}{\Ys}{z}$, called a \emph{witness} of
$t_n$, such that
\begin{itemize}[topsep=3pt]
    \item if ${\elt{\matA{\ell}}{i,j} > 0}$, then ${\elt{\x_n}{j} =
    \elt{\x}{j}}$,

    \item for each colour ${c \in \col}$ and each index ${1 \leq j \leq
    \ddim[\ell-1]}$, if ${\elt{\matB{\ell}{c}}{i,j} > 0}$, then multisets
    ${\llbrace \elt{\y}{j} \mid \y \in \Ys[c]_n \mbox{ and } \elt{\y}{j} > 0
    \rrbrace}$ and ${\llbrace \elt{\y}{j} \mid \y \in \Ys[c] \mbox{ and }
    \elt{\y}{j} > 0 \rrbrace}$ are equal.
\end{itemize}
Observe that these properties imply ${z = z_n}$. For the base case,
${\triple{\s{\ell-1}}{\Ys_\emptyset}{\Val{\ell}{i}{\s{\ell-1}}{\Ys_\emptyset}}
\in f_0}$ holds by definition, so $t_0$ is its own witness in $\F{\ell}{i}$.
For the induction step, we assume that ${t_{n-1} =
\triple{\x_{n-1}}{\Ys_{n-1}}{z_{n-1}}}$ with ${0 < n \leq K}$ has a witness in
$\F{\ell}{i}$, and we show that let ${t_n = \triple{\x_n}{\Ys_n}{z_n}}$ then
has a witness in $\F{\ell}{i}$ as well. Let ${t = \triple{\x}{\Ys}{z}}$ be a
witness of $t_{n-1}$ in $\F{\ell}{i}$. We first show that there exists ${m \in
\nnnat}$ such that $f_m$ is defined, $f_{m-1}$ contains $t$, but $f_m$ does not
contain $t$. If $\F{\ell}{i}$ is finite, the last element of $\F{\ell}{i}$ is
empty so the claim clearly holds. Thus, assume that $\F{\ell}{i}$ is infinite.
For the sake of a contradiction, assume that there exists some ${m' \in
\nnnat}$ such that $t$ appears in all elements of $\F{\ell}{i}$ after $f_{m'}$.
By the definition of $\Se{\ell}{i}$, this implies that the elements of
$\Se{\ell}{i}$ after $s_{m'}$ form an infinite sequence that is strictly
monotonic and whose values are always smaller than $\act(z)$; however, this
contradicts condition \ref{cond:finite-diverges}. Thus, there exists ${m \in
\nnnat}$ such that $f_{m-1}$ contains $t$, but $f_m$ does not. The definition
of $\Se{\ell}{i}$ ensures that the $m-1$-th element of $\Se{\ell}{i}$ is
precisely $\act(z_{n-1})$. If $z_n = z_{n-1}$, then the change from $t_{n-1}$
to $t_n$ can only take place in either the $j$-th component of $\x_{n-1}$ for
$j$ such that ${\elt{\matA{\ell}}{i,j} = 0}$, or in the $j$-th component of
some ${\y \in \Ys[c]_{n-1}}$ for some $c \in \col$ with
${\elt{\matB{\ell}{c}}{i,j} = 0}$, so the statement holds since
$\triple{\x}{\Ys}{z}$ is a witness for $t_n$. If ${z_n > z_{n-1}}$, then the
change from $t_{n-1}$ to $t_n$ can only take place in either the $j$-th
component of $\x_{n-1}$ for $j$ such that ${\elt{\matA{\ell}}{i,j} > 0}$, or in
the $j$-th component of some ${\y \in \Ys[c]_{n-1}}$ for some ${c \in \col}$
with ${\elt{\matB{\ell}{c}}{i,j} > 0}$, or by adding a new vector to some
$\Ys[c]_{n-1}$ with the smallest positive value from $\Se{\ell-1}{j}$ in the
$j$-th component, for some $j$ such that ${\elt{\matB{\ell}{c}}{i,j} > 0}$. By
the definition of a witness, both $t_{n-1}$ and $t$ agree on the components of
vectors where the change from $t_{n-1}$ to $t_n$ takes place, and so the same
change can be applied to the witness $\triple{\x}{\Ys}{z}$, leading to a triple
${t' = \triple{\x'}{\Yps}{z'}}$ such that $z' = z_n$ and by definition of
$\F{\ell}{i}$, $t'$ must appear in $f_{m+1}$. Thus, $t'$ is clearly a witness
of $t_n$ in $\F{\ell}{i}$. This concludes the proof of ($\ast$).

Now, ($\ast$) ensures that $\triple{\x_\alpha}{\Ys_\alpha}{z_\alpha}$ has a
witness in $\F{\ell}{i}$, and as we have already shown, there exists some
element $f_m$ of $\F{\ell}{i}$ such that this triple appears in $f_m$ but not
in $f_{m+1}$. But then, the definition of $\Se{\ell}{i}$ ensures that
$\act(z_\alpha)$ is the $m$-th element of $\Se{\ell}{i}$; since
${\act(z_\alpha) = \alpha}$, numner $\alpha$ appears in $\Se{\ell}{i}$, as
desired.
\end{proof}

\minexists*

\begin{proof}
By condition \ref{cond:s-x} of Lemma~\ref{lem:S}, for each ${1 \leq \ell \leq
L}$ and each ${1 \leq i \leq \ddim[\ell]}$, the elements of $\X{\ell}{i}$ are
precisely the elements of the sequence $\Se{\ell}{i}$. By condition
\ref{cond:nonnegative}, all elements in $\Se{\ell}{i}$ are nonnegative, so
${\X{\ell}{i} \subseteq \nnreal}$ holds. Moreover, assume that set
${\X{\ell}{i} \setminus \Xsub{\ell}{i}{\alpha}}$ is infinite; then, by
condition \ref{cond:s-x} of Lemma~\ref{lem:S}, set $\Se{\ell}{i}$ contains
infinitely many numbers that are smaller or equal than $\alpha$. However, by
condition \ref{cond:monotonic} ensures that $\Se{\ell}{i}$ is strictly
monotonically increasing, and so $\alpha$ is an upper bound of the sequence.
This, in turn, contradicts condition \ref{cond:finite-diverges}. Consequently,
set ${\X{\ell}{i} \setminus \Xsub{\ell}{i}{\alpha}}$ is finite.
\end{proof}

\begin{lemma}\label{lem:values-x}
    For each $(\col,\ddim)$-dataset $D$, each ${0 \leq \ell \leq L}$, each
    vector $\vlab{v}{\ell}$ labelling a vertex when $\GNN$ is applied to $D$,
    and each ${1 \leq i \leq \ddim[\ell]}$, it holds that
    ${\elt{\vlab{v}{\ell}}{i} \in \X{\ell}{i}}$.
\end{lemma}

\begin{proof}
The proof is by a straightforward induction on ${0 \leq \ell \leq L}$. For the
base case, Definition~\ref{def:canonical} ensures ${\elt{\vlab{v}{0}}{i} \in \{
0,1 \} = \X{0}{i}}$ for each $i$. For the induction step, consider some $1 \leq
\ell \leq L$ and $1 \leq i \leq \ddim[\ell]$ and notice that the value of
$\elt{\vlab{v}{\ell}}{i}$ is given by expression \eqref{eq:lambda-layer}.
Consider the triple $\triple{\x}{\Ys}{z}$ where ${\x = \vlab{v}{\ell-1}}$,
$\Ys$ is the multiset family such that, for each $c \in \col$, $\Ys[c]$ is the
multiset ${\llbrace \vlab{u}{\ell} \mid \langle v, u \rangle \in \gE{c}
\rrbrace}$, and ${z = \Val{\ell}{i}{\x}{\Ys}}$. By the inductive hypothesis,
${\elt{\x}{j} \in \X{\ell-1}{j}}$, and ${\elt{\y}{j} \in \X{\ell-1}{j}}$ for
each ${c \in \col}$ and each ${\y \in \Ys[c]}$. Finally, by comparing
\eqref{eq:lambda-layer} with the definition of $\Val{\ell}{i}{\x}{\Ys}$ in
Definition~\ref{def:possible-values}, we can see that ${z =
\elt{\vlab{v}{\ell}}{i}}$. By the definition of $\X{\ell}{i}$, then ${z \in
\X{\ell}{i}}$, and thus ${\elt{\vlab{v}{\ell}}{i} \in \X{\ell}{i}}$, as desired.
\end{proof}

\boundedequivalence*

\begin{proof}
Consider an arbitrary $(\col,\ddim)$-dataset $D$ and let ${\gG = \langle \gV,
\{\gE{c} \}_{c \in \col}, \lab \rangle}$ be the canonical encoding of $D$. Let
${\lab[0], \dots, \lab[L]}$ and ${\lab[0]', \dots, \lab[L]'}$ be the functions
labelling the vertices of $\gG$ induced by applying $\GNN$ and $\GNN'$ to $D$,
respectively.

We first prove by induction on $0 \leq \ell \leq L$ that for all ${v \in \gV}$
and ${1 \leq i \leq \ddim[\ell]}$, it holds that
${\elt{\vlab{v}{\lab[\ell]}}{i} \geq \elt{\vlab{v}{\lab[\ell]'}}{i}}$. The base
case holds trivially since ${\lab[0] = \lab[0]'}$. For the induction step,
consider ${1 \leq \ell \leq L}$, ${v \in \gV}$, and ${1 \leq i \leq
\ddim[\ell]}$, and suppose that both claims hold for $\ell-1$. The formulas for
$\elt{\vlab{v}{\lab[\ell]}}{i}$ and $\elt{\vlab{v}{\lab[\ell]'}}{i}$ are given
by equations \eqref{eq:lambda-layer} and \eqref{eq:lambdap-layer}.
\begin{align}
    \elt{\vlab{v}{\lab[\ell]}}{i} =     & \; \act \left( \sum_{j= 1}^{\ddim[\ell-1]} \elt{\matA{\ell}}{i,j} \elt{\vlab{v}{\lab[\ell-1]}}{j}  + \sum_{c \in \col}  \sum_{j=1}^{\ddim[\ell-1]} \elt{\matB{c}{\ell}}{i,j} \: \maxsum{\aggK{\ell}}    ( \llbrace \elt{\vlab{u}{\lab[\ell]}}{j}  \mid \langle v, u \rangle \in \gE{c} \rrbrace ) + \elt{\bias{\ell}}{i} \right) \label{eq:lambda-layer} \\
    \elt{\vlab{v}{\lab[\ell]'}}{i} =    & \; \act \left( \sum_{j= 1}^{\ddim[\ell-1]} \elt{\matA{\ell}}{i,j} \elt{\vlab{v}{\lab[\ell-1]'}}{j} + \sum_{c \in \col}  \sum_{j=1}^{\ddim[\ell-1]} \elt{\matB{c}{\ell}}{i,j} \: \maxsum{\capacity{\ell}}( \llbrace \elt{\vlab{u}{\lab[\ell]'}}{j} \mid \langle v, u \rangle \in \gE{c} \rrbrace ) + \elt{\bias{\ell}}{i} \right) \label{eq:lambdap-layer}
\end{align}
In both equations, all summands except $\elt{\bias{\ell}}{i}$ are nonnegative:
the weights of $\GNN$ and $\GNN'$ are nonnegative by
Definition~\ref{def:max-sum-GNN}, and the feature vectors labelling vertices of
$\gV$ are also nonnegative by Theorem~\ref{thm:min-exists} and
Lemma~\ref{lem:values-x}.
Note that the inductive hypothesis ensures ${\elt{\vlab{u}{\lab[\ell-1]}}{j}
\geq \elt{\vlab{u}{\lab[\ell-1]'}}{j}}$ for all ${u \in \gV}$ and ${1 \leq j
\leq \ddim[\ell-1]}$. Furthermore, Algorithm \ref{alg:capacity} ensures that
$\capacity{\ell} \leq \aggK{\ell}$. Since
the weights of $\matA{\ell}$ and each $\matB{\ell}{c}$ are nonnegative,
subtracting \eqref{eq:lambdap-layer} from \eqref{eq:lambda-layer} yields a
positive value, and so $\elt{\vlab{v}{\lab[\ell]}}{i} \geq
\elt{\vlab{v}{\lab[\ell]'}}{i}$. This concludes the proof by induction.

Now let $\ell_\sta$ be the largest ${1 \leq \ell \leq L}$ such that either all
elements of $\matA{\ell}$ and $\matB{\ell}{c}$ for each $c \in \col$ are $0$,
or ${\X{\ell-1}{j} = \{ 0 \}}$ for each ${1 \leq j \leq \ddim[\ell-1]}$; if
such $\ell$ does not exist, let ${\ell_\sta = 0}$. To complete the proof of the
theorem, we prove by induction on ${\ell_\sta \leq \ell \leq L}$ that, for all
${v \in \gV}$ and ${1 \leq i \leq \ddim[\ell]}$, exactly one of the following
two properties holds:
\begin{itemize}[topsep=3pt,itemsep=3pt]
    \item ${\elt{\vlab{v}{\lab[\ell]}}{i} = \elt{\vlab{v}{\lab[\ell]'}}{i}}$ or

    \item ${\elt{\vlab{v}{\lab[\ell]}}{i} > \elt{\vlab{v}{\lab[\ell]'}}{i} \geq \afteract{\ell}}$.
\end{itemize}
For the base case, if ${\ell_\sta = 0}$, then the first property holds
trivially since ${\lab[0] = \lab[0]'}$. If ${\ell_\sta > 0}$, consider an
arbitrary ${1 \leq i \leq \ddim[\ell]}$. We have two possibilities. First, if
all elements of $\matA{\ell_\sta}$ and $\matB{\ell_\sta}{c}$ for each $c \in
\col$ are all zero, equations \eqref{eq:lambda-layer} and
\eqref{eq:lambdap-layer} and the fact that the matrices of $\GNN$ and $\GNN'$
are the same ensure that ${\elt{\vlab{v}{\lab[\ell_\sta]}}{i} =
\elt{\vlab{v}{\lab[\ell_\sta]'}}{i} = \act(\elt{\bias{\ell_\sta}}{i})}$.
Second, if $\X{\ell-1}{j} = \{ 0\}$ for each ${1 \leq j \leq \ddim[\ell-1]}$;
equation \eqref{eq:lambda-layer} ensure ${\elt{\vlab{v}{\lab[\ell_\sta]}}{i} =
\act(\elt{\bias{\ell_\sta}}{i})}$; moreover, we have shown that
${\elt{\vlab{v}{\lab[\ell_\sta]}}{i} \geq
\elt{\vlab{v}{\lab[\ell_\sta']}}{i}}$, and since the elements in the sum in
\eqref{eq:lambdap-layer} other than $\elt{\bias{\ell}}{i}$ are not negative, we
again have ${\elt{\vlab{v}{\lab[\ell_\sta]'}}{i} =
\act(\elt{\bias{\ell_\sta}}{i})}$. Hence, ${\elt{\vlab{v}{\lab[\ell]}}{i} =
\elt{\vlab{v}{\lab[\ell]'}}{i}}$ and the first property holds.

For the induction step, we consider arbitrary layer ${\ell_\sta < \ell \leq
L}$, vertex ${v \in \gV}$, and position ${1 \leq i \leq \ddim[\ell]}$. We
assume that ${\elt{\vlab{v}{\lab[\ell]}}{i} \neq
\elt{\vlab{v}{\lab[\ell]'}}{i}}$; together with ${\elt{\vlab{v}{\lab[\ell]}}{i}
\geq \elt{\vlab{v}{\lab[\ell]'}}{i}}$, this implies
${\elt{\vlab{v}{\lab[\ell]}}{i} > \elt{\vlab{v}{\lab[\ell]'}}{i}}$, so we next
show ${\elt{\vlab{v}{\lab[\ell]'}}{i} \geq \afteract{\ell}}$. Since ${\ell >
\ell_\sta}$, Algorithm \ref{alg:capacity} defines $\minX{\ell}$,
$\beforeact{\ell}$, $\minW{\ell}$, $\minB{\ell}$, $\capacity{\ell}$,
$\afteract{\ell}$, and $\afteract{\ell-1}$. Furthermore, let ${\aggK{\ell}' =
\lceil \frac{\beforeact{\ell} - \minB{\ell}}{\minW{\ell} \cdot \minX{\ell}}
\rceil}$, so ${\capacity{\ell} = \min(\aggK{\ell}, \aggK{\ell}')}$. We have
already shown that ${\elt{\vlab{u}{\lab[\ell-1]}}{j} \geq
\elt{\vlab{u}{\lab[\ell-1]'}}{j}}$ for all ${u \in \gV}$ and ${1 \leq j \leq
\ddim[\ell-1]}$. We next consider the following four possibilities.

\smallskip

\textbf{Case 1.} There exists ${1 \leq j \leq \ddim[\ell-1]}$ such that
${\elt{\matA{\ell}}{i,j} > 0}$ and $\elt{\vlab{v}{\lab[\ell-1]'}}{j} \geq
\afteract{\ell-1}$. Since all summands in \eqref{eq:lambdap-layer} except
$\elt{\bias{\ell}}{i}$ are nonnegative, the argument of $\act$ in
\eqref{eq:lambdap-layer} is greater or equal to $ \elt{\matA{\ell}}{i,j}
\elt{\vlab{v}{\lab[\ell-1]'}}{j} + \elt{\bias{\ell}}{i} \geq \minW{\ell}
\afteract{\ell-1} + \minB{\ell} = \beforeact{\ell}$; since
$\act(\beforeact{\ell}) \geq \afteract{\ell}$ and $\act$ is monotonically
increasing, we have $\elt{\vlab{v}{\lab[\ell]'}}{i} \geq \alpha_\ell$, as
desired.

\smallskip

\textbf{Case 2.} Case 1 does not hold and ${\capacity{\ell} = 0}$. If
${\capacity{\ell} = \aggK{\ell} = 0}$, the sum over ${c \in \col}$ in both
\eqref{eq:lambda-layer} and \eqref{eq:lambdap-layer} is always equal to $0$.
Furthermore, since case 1 does not hold, the induction hypothesis ensures that
for any $1 \leq j \leq \ddim[\ell-1]$ such that $A^{\ell}_{i,j} > 0$, we have
${\elt{\vlab{v}{\lab[\ell-1]'}}{j} = \elt{\vlab{v}{\lab[\ell-1]}}{j}}$. Thus,
it follows that $\elt{\vlab{v}{\lab[\ell]}}{i} = \elt{\vlab{v}{\lab[\ell]'}}{i}
$. If $\capacity{\ell} = \aggK{\ell}' = 0$, Algorithm \ref{alg:capacity}
ensures that $\beforeact{\ell} = \minB{\ell}$. Then, since all summands in the
argument of $\act$ other than $\elt{\bias{\ell}}{i}$ are nonnegative, we have
that the argument of $\act$ in \eqref{eq:lambdap-layer} is greater or equal
than $\elt{\bias{\ell}}{i} \geq \minB{\ell} = \beforeact{\ell}$, and since
$\act$ is monotonically increasing and $\act(\beforeact{\ell}) \geq
\afteract{\ell}$, we have that $\elt{\vlab{v}{\lab[\ell]'}}{i} \geq
\alpha_\ell$, as desired.

\smallskip

\textbf{Case 3.} $\capacity{\ell} > 0$ and there exist ${c \in \col}$, ${1 \leq
j \leq \ddim[\ell-1]}$, and $\langle v, u \rangle \in \gE{c}$ such that
$\elt{\matB{c}{\ell}}{i,j} > 0$ and $\elt{\vlab{u}{\lab[\ell-1]'}}{j} \geq
\afteract{\ell-1}$. All summands in the argument of $\act$ in
\eqref{eq:lambdap-layer} except $\elt{\bias{\ell}}{i}$ are nonnegative and
${\maxsum{\capacity{\ell}}(\llbrace \elt{\vlab{u}{\lab[\ell-1]'}}{j} \mid
\langle v, w \rangle \in \gE{c} \rrbrace) \geq
\elt{\vlab{u}{\lab[\ell-1]'}}{j}}$ due to ${\capacity{\ell} > 0}$, so the
argument of $\act$ in \eqref{eq:lambdap-layer} is greater or equal to
${\elt{\matB{c}{\ell}}{i,j} \elt{\vlab{u}{\lab[\ell-1]'}}{j} +
\elt{\bias{\ell}}{i} \geq \minW{\ell} \afteract{\ell-1} + \elt{\bias{\ell}}{i}
= \beforeact{\ell}}$. Since ${\act(\beforeact{\ell}) \geq \afteract{\ell}}$ and
$\act$ is monotonically increasing, we have ${\elt{\vlab{v}{\lab[\ell]'}}{i}
\geq \alpha_\ell}$.

\smallskip

\textbf{Case 4.} None of cases 1--3 hold. Since case 1 does not hold, the
induction hypothesis ensures that for any $1 \leq j \leq \ddim[\ell-1]$ such
that $A^{\ell}_{i,j} > 0$, we have $\elt{\vlab{v}{\lab[\ell-1]'}}{j} =
\elt{\vlab{v}{\lab[\ell-1]}}{j} $. Furthermore, since case 2 does not hold, we
have ${\capacity{\ell} > 0}$. Finally, case 3 does not hold, so, for each ${c
\in \col}$ and ${1 \leq j \leq \ddim[\ell-1]}$ such that
${\elt{\matB{\ell}{c}}{i,j} > 0}$, we have $\elt{\vlab{u}{\lab[\ell-1]'}}{j} =
\elt{\vlab{u}{\lab[\ell-1]}}{j}$ for each $u$ such that ${\langle v, u \rangle
\in \gE{c}}$, and so ${\llbrace \elt{\vlab{u}{\lab[\ell]}}{j} \mid \langle v, u
\rangle \in \gE{c} \rrbrace = \llbrace \elt{\vlab{u}{\lab[\ell]'}}{j} \mid
\langle v, u \rangle \in \gE{c} \rrbrace}$ holds. By these observations, our
assumption that ${\elt{\vlab{v}{\lab[\ell]}}{i} \neq
\elt{\vlab{v}{\lab[\ell]'}}{i}}$, and equations \eqref{eq:lambda-layer} and
\eqref{eq:lambdap-layer}, then ${\capacity{\ell} = \aggK{\ell}' < \aggK{\ell}}$
and there must exist at least one ${1 \leq j \leq \ddim[\ell-1]}$ such that
${\elt{\matB{\ell}{c}}{i,j} > 0}$ and the number of distinct $u$ such that
${\langle v, u \rangle \in \gE{c}}$ and ${\elt{\vlab{u}{\lab[\ell-1]}}{j} > 0}$
is greater than $\capacity{\ell}$. For such $j$, and since all summands in the
argument of $\sigma$ in \eqref{eq:lambdap-layer} except $\elt{\bias{\ell}}{i}$
are nonnegative, it holds that the argument is greater or equal than
\begin{align}
    \elt{\matB{\ell}{c}}{i,j} \maxsum{\capacity{\ell}} (\llbrace \elt{\vlab{u}{\lab[\ell]'}}{j} \mid \langle v, u \rangle \in \gE{c} \rrbrace)+ \elt{\bias{\ell}}{i}.   \label{eq:single-j}
\end{align}
However, as we have already observed, we know that there exist at least
$\capacity{\ell}$ elements different from zero in the multiset in
\eqref{eq:single-j}, and Lemma~\ref{lem:values-x} and the definitions of
$\minX{\ell}$ amd $\X{\ell-1}{j}$ ensure that each of these elements is greater
or equal than $\minX{\ell}$. Thus, the value in \eqref{eq:single-j} is greater
or equal than $ \minW{\ell} \capacity{\ell} \minX{\ell} + \minB{\ell} \geq
\beforeact{\ell}$. However, ${\act(\beforeact{\ell}) \geq \afteract{\ell}}$
since $\act$ is monotonic, we have ${\elt{\vlab{v}{\lab[\ell]}}{i} \geq
\afteract{\ell}}$, which concludes the proof.

To complete the proof of the theorem, we consider an arbitrary term $t$ in $D$
and an arbitrary unary predicate $U_i$ in the $(\col,\ddim)$-signature, where
$1 \leq i \leq \ddim$, and we show that $U_i(t) \in T_{\GNN}(D)$ if and only if
$U_i(t) \in T_{\GNN'} (D)$; this implies the theorem since $T_{\GNN}(D)$ and
$T_{\GNN'} (D)$ can only contain atoms of this form. By definition of the
canonical encoder/decoder scheme and the definitions of both $\GNN$ and
$\GNN'$, it suffices to show that
\begin{equation}
    \cls(\elt{\vlab{v}{\lab[L]}}{i}) = 1 \mbox{ if and only if } \cls(\elt{\vlab{v}{\lab[L]'}}{i}) = 1,   \label{eq:equivalence}
\end{equation}
for $v$ the vertex of the form $v_t$ in $\gV$. By the first result shown above
by induction, we have that $\elt{\vlab{v}{\lab[L]}}{i} \geq
\elt{\vlab{v}{\lab[L]'}}{i}$, and the second result ensures that either
$\elt{\vlab{v}{\lab[L]}}{i} = \elt{\vlab{v}{\lab[L]'}}{i}$ or
$\elt{\vlab{v}{\lab[L]'}}{i} \geq \afteract{L}$, since $L \geq \ell_\sta$. If
$\elt{\vlab{v}{\lab[L]}}{i} = \elt{\vlab{v}{\lab[L]'}}{i}$,
\eqref{eq:equivalence} holds trivially. If $\elt{\vlab{v}{\lab[L]'}}{i} \geq
\afteract{L}$, then the definition of $\afteract{L}$ in Algorithm
\ref{alg:capacity} ensures that $\cls\elt{(\vlab{v}{\lab[L]'}}{i})=1$, and
since $\elt{\vlab{v}{\lab[L]}}{i} \geq \elt{\vlab{v}{\lab[L]'}}{i}$, then
$\cls(\elt{\vlab{v}{\lab[L]}}{i}) = 1$, so \eqref{eq:equivalence} holds.
\end{proof}

For a dataset $D$, let $\tms{D}$ be the set containing each term $t$ such that
$D$ contains an atom of the form $U(t)$, $\edg{c}(t,s)$, or $\edg{c}(s,t)$, for
$U$ and $\edg{c}$ arbitrary unary and binary predicates, respectively, and $s$
an arbitrary term. An \emph{isomorphism} from a $(\col,\ddim)$-dataset $D$ to a
$(\col,\ddim)$-dataset $D'$ is an injective mapping $h$ of terms to terms that
is defined (at least) on all $\tms{D}$ and satisfies ${h(D) = D'}$, where
$h(D)$ is the dataset obtained by replacing each fact of the form $U(t)$ in $D$
with $U(h(t))$, and each fact of the form $\edg{c}(t,s) \in D$ with
$\edg{c}(h(t),h(s))$.

\begin{lemma}\label{lem:invariance}
    For all $(\col,\ddim)$ datasets $D$ and $D'$, the following properties holds:
    \begin{enumerate}[start=1,label={(M\arabic*)},leftmargin=30pt]
        \item\label{prop:invariance}
        each isomorphism from $D$ to $D'$ is also an isomorphism from
        $T_\GNN(D)$ to $T_\GNN(D')$; and

        \item\label{prop:monotonicity}
        ${D \subseteq D'}$ implies ${T_\GNN(D) \subseteq T_\GNN(D')}$.
    \end{enumerate}
\end{lemma}

\begin{proof}
It is straightforward to see that property \ref{prop:invariance} holds: for any
two $(\col,\ddim)$-datasets, an isomorphism $h$ from $D$ to $D'$ induces a
bijective mapping between the vertices of $\canenc{D}$ and $\canenc{D'}$;
moreover, the result of applying $\GNN$ to a $(\col,\ddim)$-graph depends only
on the graph structure and not on the vertex names, so it is straightforward to
show that the vectors labelling the corresponding vertices are identical.

To see that property \ref{prop:monotonicity} holds, consider arbitrary datasets
$D$ and $D'$ such that ${D \subseteq D'}$. Let ${\gG = \langle \gV, \{ \gE{c}
\}_{c \in \col}, \lab \rangle}$ and ${\gG' = \langle \gV', \{ \gEp{c} \}_{c \in
\col}, \lab' \rangle}$ be the canonical encodings of $D$ and $D'$,
respectively, and ${\lab[0], \dots, \lab[L]}$ and ${\lab[0]', \dots, \lab[L]'}$
be the functions labelling the vertices of $\gG$ and $\gG'$ when $\GNN$ is
applied to these graphs. By a straightforward induction on ${0 \leq \ell \leq
L}$ we show that ${\elt{\vlab{v}{\lab[\ell]}}{i} \leq
\elt{\vlab{v}{\lab[\ell]'}}{i}}$ holds for each vertex ${v \in \gV}$ and each
${1 \leq i \leq \ddim[\ell]}$. The base case for ${\ell = 0}$ follows
immediately from the canonical encoding and the fact that ${D \subseteq D'}$.
For the induction step, the canonical encoding and ${D \subseteq D'}$ imply
${\gE{c} \subseteq \gEp{c}}$. The values of $\elt{\vlab{v}{\lab[\ell]}}{i}$ and
$\elt{\vlab{v}{\lab[\ell]'}}{i}$ are computed by equation
\eqref{eq:GNN-propagation}. Now by the inductive hypothesis,
${\elt{\vlab{u}{\lab[\ell-1]}}{j} \leq \elt{\vlab{u}{\lab[\ell]'}}{j}}$ holds
for each ${u \in \gV}$ and ${1 \leq j \leq \ddim[\ell-1]}$, which ensures
\begin{displaymath}
    \maxsum{\aggK{\ell}}(\llbrace \elt{\vlab{u}{\lab[\ell-1]}}{j} \mid \langle v, u \rangle \in \gE{c} \rrbrace) \leq \maxsum{\aggK{\ell}}(\llbrace \elt{\vlab{u}{\lab[\ell-1]'}}{j} \mid \langle v, u \rangle \in \gEp{c} \rrbrace).
\end{displaymath}
All elements of $\matA{\ell}$ and all $\matB{\ell}{c}$ with ${c \in \col}$ are
nonnegative, and $\act$ is monotonically increasing, which implies
${\elt{\vlab{v}{\lab[\ell]}}{i} \leq \elt{\vlab{v}{\lab[\ell]'}}{i}}$. Finally,
$\cls$ is a step function, so ${\cls(\elt{\vlab{v}{\lab[\ell]}}{i}) \leq
\cls(\elt{\vlab{v}{\lab[\ell]'}}{i})}$ holds as well, which ensures ${T_\GNN(D)
\subseteq T_\GNN(D')}$.
\end{proof}

\capturedruletest*

\begin{proof}
If there exists a substitution ${\nu : V \to S}$ such that ${H\nu \in T_r(A
\nu)}$ but ${H\nu \notin T_\GNN(A\nu)}$, then by definition $T_\GNN$ does not
capture $r$. To conclude the proof of the lemma, it only remains to show the
converse implication: if ${H\nu \in T_\GNN(A\nu)}$ for each substitution ${\nu
: v \to S}$ such that ${H\nu \in T_r(A\nu)}$, then $T$ captures $r$. To this
end, we consider an arbitrary $(\col,\ddim)$-dataset $D$, and we prove that
${T_r(D) \subseteq T_\GNN(D)}$. If $T_r(D)$ is empty, then the claim holds
vacuously, so suppose $T_r(D) \neq \emptyset$. Consider an arbitrary element
$\alpha$ in $T_r(D)$; clearly, $\alpha$ is of the form $H\mu$ for some
substitution $\mu$ such that ${A\mu \subseteq D}$ and ${H\mu \in T_r(A\mu)}$.
Let $h$ be an injective mapping from $\tms{A \mu}$ to the constants in $S$;
such a mapping exists because the body of $r$ contains at most $|V|$ variables,
and so $\tms{A \mu}$ contains at most $|V|$ terms. Then, ${\nu = h \circ \mu}$
is a substitution mapping all variables in $r$ to constants in $|V|$. Mapping
$h$ is injective, rule $r$ is constant-free, and ${H\mu \in T_r(A\mu)}$, so the
semantics of Datalog rule application ensure that ${h(H\mu) \in T_r(h(A\mu))}$,
and so ${H\nu \in T_r(A\nu)}$. The latter implies ${H\nu \in T_\GNN(A\nu)}$ by
the lemma assumption. Moreover, $h$ is an isomorphism from $A\mu$ to $A\nu$, so
property \ref{prop:invariance} of Lemma \ref{lem:invariance} implies ${H\mu \in
T_\GNN(A\mu)}$. Finally, property \ref{prop:monotonicity} of
Lemma~\ref{lem:invariance} and ${A\mu \subseteq D}$ imply ${\alpha = H\mu \in
T_\GNN(D)}$, as required.
\end{proof}

\equivalence*

\begin{proof}
We prove the theorem by showing that ${T_\GNN(D) = T_{\Prog_\GNN}(D)}$ holds
for each $(\col,\delta)$-dataset $D$. GNN $\GNN$ captures every rule in
$\Prog_\GNN$ and thus ${T_r(D) \subseteq T_\GNN(D)}$ for each ${r \in
\Prog_\GNN}$; since ${T_{\Prog_\GNN}(D) = \bigcup_{r \in \Prog_\GNN} T_r(D)}$,
we have ${T_{\Prog_\GNN}(D) \subseteq T_\GNN(D)}$.

To prove ${T_\GNN(D) \subseteq T_{\Prog_\GNN}(D)}$, we consider an arbitrary
fact ${\alpha \in T_\GNN(D)}$, and we construct a ${(L,|\col| \cdot \ddim[\GNN]
\cdot \capacity{\GNN})}$-tree-like rule $r$ such that ${\alpha \in T_r(D)}$ and
$r$ is captured by $T_\GNN$, which together imply $\alpha \in
T_{\Prog_\GNN}(D)$. To find $r$, we consider the GNN $\GNN'$ obtained from
$\GNN$ by replacing $\aggK{\ell}$ with $\capacity{\ell}$ for each ${1 \leq \ell
\leq L}$. Theorem \ref{thm:bounded-equivalence} ensures $T_\GNN(D) =
T_{\GNN'}(D)$, and so ${\alpha \in T_{\GNN'}(D)}$. Let ${\gG = \langle \gV, \{
\gE{c} \}_{c \in \col}, \lab \rangle}$ be the canonical encoding of $D$, and
let ${\lab[0], \dots, \lab[L]}$ be the functions labelling the vertices of
$\gG$ when $\GNN'$ is applied to it. We next construct an atom $H$, a
conjunction $\Gamma$, a substitution $\nu$ from the variables in $\Gamma$ to
$\tms{D}$, a graph $U$ (without vertex labels) with fresh vertices not
occurring in $\gG$ of the form $u_x$ for $x$ a variable and edges with colours
in $\col$, and mappings ${M_{c,\ell,j} : U \to 2^{\gV}}$ for each ${c \in
\col}$, ${1 \leq \ell \leq L}$, and ${1 \leq j \leq \ddim[\ell-1]}$. We also
assign to each vertex in $U$ a level between $0$ and $L$, and we identify a
single vertex from $U$ as the \emph{root} vertex. In the rest of this proof, we
use letters $t$ and $s$ for terms in $\tms{D}$, letters $x$ and $y$ for
variables, letters $v$, $w$ for the vertices in $\gV$, and (possibly indexed)
letter $u$ for the vertices in $U$. Our construction is by induction from level
$L$ down to level $1$. The base case defines a vertex of level $L$. Then, for
each $1 \leq \ell \leq L$, the induction step considers the vertices of level
$\ell$ and defines new vertices of level $\ell-1$.

We initialise $\Gamma$ as the empty conjunction, and we initialise $\nu$ and
each $M_{c,\ell,j}$ as the empty mappings. For the base case, we note that
$\alpha$ must be of the form $U_i(t)$, and so $\gV$ contains a vertex $v_t$. We
introduce a fresh variable $x$, and define ${\nu(x) = t}$; we define ${H =
U_i(x)}$; we introduce vertex $u_x$ of level $L$; and we make $u_x$ the root
vertex. Finally, we extend $\Gamma$ with atom $U(x)$ for each ${U(t) \in D}$.
For the induction step, consider $1 \leq \ell \leq L$ and assume that all
vertices of level greater than $\ell$ have been already defined. We then
consider each vertex of the form $u_x$ of level $\ell$. Let ${t = \nu(x)}$. For
each colour $c \in \col$, each layer $1 \leq \ell' < \ell$, and each dimension
$j \in \{ 1, \dots, \ddim[\ell'-1] \}$, let
\begin{align}
    M_{c,\ell',j}(u_x) = \left\{ w \mid \langle v_t, w \rangle \in \gE{c} \text{ and } \elt{\vlab{w}{\lab[\ell']}}{j} \text{ contributes to the result of } \maxsum{\capacity{\ell'}}\big(\llbrace \elt{\vlab{w}{\lab[\ell']}}{j} \mid \langle v_t, w \rangle \in \gE{c} \rrbrace\big) \right\}.   \label{eq:sum-neighbour}
\end{align}
At least one such set exists, but it may not be unique; however, any set
satisfying \eqref{eq:sum-neighbour} can be chosen. Each vertex of
$M_{c,\ell',j}(u_x)$ must be of the form $v_{s_n}$ for some term $s_n \in
\tms{D}$, where $s_n \neq s_m$ for all $1 \leq n < m \leq
|M_{c,\ell',j}(u_x)|$. We then introduce a fresh variable $y_n$ and define
${\nu(y_n) = s_n}$; we introduce a vertex $u_{y_n}$ of level ${\ell-1}$ and an
edge $\edg{c}(u_x,u_{y_n})$ to $U$; and we append to $\Gamma$ the conjunction
\begin{equation}
    \bigwedge_{n=1}^{|W|} \Big( \edg{c}(x,y_n) \wedge B_D(y_n) \Big) \wedge \bigwedge_{1 \leq n < m \leq |W|} y_n \noteq y_m,   \label{eq:append-gamma}
\end{equation}
where ${W = M_{c,\ell',j}(u_x)}$ and $B_D(y_n)$ is the conjunction consisting
of an atom $U(y_n)$ for each ${U(s_n) \in D}$. Since each $M_{c,\ell,j}(u_x)$
contains at most $\capacity{\ell'}$ elements, this step adds at most $|\col|
\cdot \ddim[\ell'-1] \cdot \capacity{\ell'} \cdot \ell'$ new successors of
$u_x$. This completes our inductive construction. At this point, $H = U_i(x)$
and $\Gamma$ is a $(L,|\col| \cdot \ddim[\GNN] \cdot
\capacity{\GNN})$-tree-like formula for $x$. Thus, rule ${H \leftarrow \Gamma}$
is a $(L,|\col| \cdot \ddim[\GNN] \cdot \capacity{\GNN})$-tree-like rule.
Furthermore, the construction of $\nu$ ensures $D \models \Gamma \nu$ so $H \nu
\in T_r(D)$, but $H \nu = \alpha$, so $\alpha \in T_r(D)$, as required.

To complete the proof, we next show that $r$ is captured by $T_\GNN$, which is
equivalent to showing that $r$ is captured by $T_{\GNN'}$. To do this, we
consider an arbitrary dataset $D'$ and an arbitrary ground atom $\alpha'$ such
that $\alpha' \in T_r(D')$. This implies that there exists some substitution
$\nu'$ such that $D' \models \Gamma \nu'$ and and $\alpha' = H \nu'$. Consider
the encoding of $D'$ into a $(\col,\ddim)$-graph ${\gG = \langle \gV',\{
\gEp{c} \}_{c \in \col}, \lab' \rangle}$, and let ${\lab[0]', \dots, \lab[L]'}$
be the functions labelling the vertices of $\gG'$ when $\GNN'$ is applied to
it. We use letters $p$, $q$, and $q'$ for the vertices of $\gV'$.

We now prove the following statement by induction: for each ${0 \leq \ell \leq
L}$ and each vertex $u_x$ of $U$ whose level is at least $\ell$, we have
$\elt{\vlab{v}{\lab[\ell]}}{i} \leq \elt{\vlab{p}{\lab[\ell]'}}{i}$ for each $i
\in \{1, \dots, \ddim[\ell]\}$, where $v=v_{\nu(x)}$ and $p=v_{\nu'(x)}$. For
the base case, ${\ell = 0}$, consider an arbitrary $1 \leq i \leq \ddim[0]$ and
$u_x \in U$, and let $v=v_{\nu(x)}$ and $p=v_{\nu'(x)}$. Note that
$\elt{\vlab{v}{\lab[0]}}{i} \in\{ 0, 1 \}$ and $(\vlab{p}{\lab[0]'}) \in\{0,1\}$,
so we only need to prove that $\elt{\vlab{v}{\lab[0]}}{i} = 1$ implies
$\elt{\vlab{p}{\lab[0]'}}{i} = 1$. By Definition~\ref{def:canonical},
$\elt{\vlab{v}{\lab[0]}}{i} = 1$ implies $U_i(x\nu) \in D$. The construction of
$\Gamma$ ensures that ${U_i(x) \in \Gamma}$, and $D' \models \Gamma \nu'$
implies $U_i(x \nu') \in D'$, and so $\elt{\vlab{p}{\lab[0]}}{i} = 1$, as
required.

For the induction step, assume that the property holds for some ${\ell-1}$, and
consider an arbitrary vertex ${u_x \in U}$ whose level is at least $\ell$;
consider an arbitrary $c \in \col$, $i \in \{ 1, \dots, \ddim[i] \}$, and let
$v=v_{\nu(x)}$ and $p=v_{\nu'(x)}$. Note that the following holds.
\begin{align}
    \elt{\vlab{v}{\lab[\ell]}}{i}  = \, & \act \left( \sum_{j=1}^{\ddim[\ell-1]} \elt{\matA{\ell}}{i,j} \elt{\vlab{v}{\lab[\ell-1]}}{j}  + \sum_{c \in \col} \sum_{j=1}^{\ddim[\ell-1]} \elt{\matB{\ell-1}{c}}{i,j} \; \maxsum{\capacity{\ell}} \llbrace \elt{\vlab{w}{\lab[\ell-1]}}{j}  \mid \langle v, w \rangle \in \gE{c}  \rrbrace + \elt{\bias{\ell}}{i} \right) \label{eq:propagation-sum-G} \\
    \elt{\vlab{p}{\lab[\ell]'}}{i} = \, & \act \left( \sum_{j=1}^{\ddim[\ell-1]} \elt{\matA{\ell}}{i,j} \elt{\vlab{p}{\lab[\ell-1]'}}{j} + \sum_{c \in \col} \sum_{j=1}^{\ddim[\ell-1]} \elt{\matB{\ell-1}{c}}{i,j} \; \maxsum{\capacity{\ell}} \llbrace \elt{\vlab{q}{\lab[\ell-1]'}}{j} \mid \langle p, q \rangle \in \gEp{c} \rrbrace + \elt{\bias{\ell}}{i} \right) \label{eq:propagation-sum-G'}
\end{align}
The induction assumption ensures ${\elt{\vlab{v}{\lab[\ell-1]}}{j} \leq
\elt{\vlab{p}{\lab[\ell-1]'}}{j}}$ for each ${1 \leq j \leq \ddim[\ell-1]}$.
Also, for each colour ${c \in \col}$ and each ${1 \leq j \leq \ddim[\ell-1]}$,
we have that $\maxsum{\capacity{\ell}} \llbrace \elt{\vlab{w}{\lab[\ell-1]}}{j}
\mid \langle v, w \rangle \in \gE{c} \rrbrace$ is equal to ${\sum_{w \in W}
\elt{\vlab{w}{\lab[\ell-1]}}{j}}$, where ${W = M_{c,\ell-1,j}(u_x)}$. Recall
that the elements of $W$ are of the form ${v_{s_1}, \cdots, v_{s_{|W|}}}$ where
${s_1, \cdots, s_{|W|}}$ are terms in $\tms{D}$. Furthermore, by the
construction of $U$, there are $W$ distinct vertices ${u_{y_1}, \cdots,
u_{y_{|W|}}}$ in $U$ of level $\ell-1$ such that ${\nu(y_n) = s_n}$ and
$\edg{c}(u_x,u_{y_n})$ is in $U$ for each ${1 \leq n \leq |W|}$. Furthermore,
$\Gamma$ contains atoms ${\edg{c}(x,y_1), \dots, \edg{c}(x,y_{|W|})}$ as well
as inequalities ${y_n \noteq y_m}$ for ${1 \leq n < m \leq |W|}$. We then have
${\edg{c}(\nu'(x),\nu'(y_n)) \in D'}$ and ${\nu'(y_n) \neq \nu'(y_m)}$ for ${1
\leq n < m \leq |W|}$. Thus, ${W' = \{ v_{\nu'(y_1)}, \dots, v_{\nu'(y_{|W|})}
\}}$ is a set of $|W|$ \emph{distinct} $c$-neighbours of $v_{\nu(x)}$ in
$\gG'$. The induction assumption ensures that ${w = v_{\nu(y_n)}}$ and ${q =
v_{\nu'(y_n)}}$ imply ${\elt{\vlab{w}{\lab[\ell-1]}}{j} \leq
\elt{\vlab{q}{\lab[\ell-1]'}}{j}}$, and so ${\sum_{w \in W}
\elt{\vlab{w}{\lab[\ell-1]}}{j} \leq \sum_{q \in W'}
\elt{\vlab{q}{\lab[\ell-1]'}}{j}}$. Thus, by equations
\eqref{eq:propagation-sum-G} and \eqref{eq:propagation-sum-G'}, the fact that
the elements from $\matA{\ell}$ and all $\matB{\ell}{c}$ are nonnegative, and
$\act$ is monotonically increasing, we have $\elt{\vlab{v}{\lab[\ell]}}{i} \leq
\elt{\vlab{p}{\lab[\ell]'}}{i}$, as required.

Recall that ${\alpha' = H\nu'}$ is of the form $U_i(t')$ with ${t' = \nu'(x)}$;
moreover, ${U_i(t) \in T_{\GNN'}(D)}$ with ${t = \nu(x)}$. Now let $v=v_t$ and
$p = v_{t'}$. Now ${U_i(t) \in T_{\GNN'}(D)}$ implies
$\cls(\elt{\vlab{v}{\lab[L]}}{i}) = 1$, and the above property ensures
$\elt{\vlab{v}{\lab[L]}}{i} \leq \elt{\vlab{p}{\lab[L]'}}{i}$; since $\cls$ is
a step function, we have $\cls(\elt{\vlab{p}{\lab[L]'}}{i}) = 1$. Hence,
$U_i(t') \in T_{\GNN'}(D')$, as required.
\end{proof}

\algcapacity*

\begin{proof}
We first prove the two items of the theorem, and then we prove that
Algorithm~\ref{alg:capacity} terminates.

First, recall that by condition \ref{cond:s-x} of Lemma~\ref{lem:S}, for each
${0 \leq \ell \leq L}$ and each $1 \leq i \leq \ddim[\ell]$, set $\X{\ell}{i}$
is contains exactly all the elements of $\Se{\ell}{i}$. Furthermore, since
$\Se{\ell}{i}$ is strictly monotonically increasing by condition
\ref{cond:monotonic} of Lemma \ref{lem:S}, its smallest element is its first
element. Hence, the smallest element of $\X{\ell}{i}$ is the first element of
$\Se{\ell}{i}$. Furthermore, for any $\alpha \in \real$, let
$\Ssub{\ell}{i}{\alpha}$ be the subsequence of $\Se{\ell}{i}$ which contains
all elements in $\Ssub{\ell}{i}{\alpha}$ greater than $\alpha$. Clearly,
$\Xsub{\ell}{i}{\alpha}$ is identical to the set of elements in
$\Ssub{\ell}{i}{\alpha}$. Furthermore, since $\Se{\ell}{i}$ is strictly
monotonically increasing, then either $\Ssub{\ell}{i}{\alpha}$ is empty or it
contains an element $\ssub{\ell}{i}{\alpha}$ which appears in $\Se{\ell}{i}$
exactly once and satisfies the following conditions:
\begin{enumerate}[label={(A\arabic*)},leftmargin=30pt]
    \item\label{cond:A:smaller}
    all elements that precede $\ssub{\ell}{i}{\alpha}$ in $\Se{\ell}{i}$ are
    smaller or equal to $\alpha$; and

    \item\label{cond:A:bigger}
    all elements that follow $\ssub{\ell}{i}{\alpha}$ in $\Se{\ell}{i}$ are
    strictly greater than $\ssub{\ell}{i}{\alpha}$.
\end{enumerate}
In particular, condition \ref{cond:A:bigger} ensures that if
$\Ssub{\ell}{i}{\alpha}$ is not empty, then $\ssub{\ell}{i}{\alpha}$ is its
smallest element. Hence, to show the items of the theorem, it suffices to prove
the following:
\begin{itemize}
    \item $\Next{\ell}{i}{\symstart}$ returns the first element of
    $\Se{\ell}{i}$, and

    \item for each $\alpha \in \real$, $\Next{\ell}{i}{\alpha}$ returns
    $\symend$ if $\Ssub{\ell}{i}{\alpha}$ is empty, and otherwise it returns
    $\ssub{\ell}{i}{\alpha}$.
\end{itemize}
We show both items simultaneously via induction over $0 \leq \ell \leq L$.

For the base case $\ell=0$, consider an arbitrary $1 \leq i \leq \ddim[0]$. To
see that $\Next{0}{i}{\symstart}$ returns the first element of $\Se{0}{i}$,
simply note that line \ref{alg:next:ell0:first} of Algorithm \ref{alg:next}
ensures that $\Next{0}{i}{\symstart} = 0$, which is precisely the smallest
element of $\Se{0}{i}$. To prove the second item, consider an arbitrary $\alpha
\in \real$. If $\Ssub{0}{i}{\alpha}$ is empty then, $\alpha \geq 1$, in which
case line \ref{alg:next:ell0:third} of Algorithm \ref{alg:next} ensures
$\Next{0}{i}{\alpha} = \symend$, as expected. If $\Ssub{0}{i}{\alpha}$ is not
empty, we consider two possible cases: $\alpha <0 $ or $0 \leq \alpha <1$. If
$\alpha < 0$, then $\Ssub{0}{i}{\alpha} = (0,1)$, but $\Next{0}{i}{\alpha} =0$
by line \ref{alg:next:ell0:first} of Algorithm \ref{alg:next}, so the claim
holds. If $0 \leq \alpha < 1$, then $\Ssub{0}{i}{\alpha} = (1)$. But then, line
\ref{alg:next:ell0:second} of Algorithm \ref{alg:next} ensures
$\Next{0}{i}{\alpha} =1$.

For the induction step, consider some arbitrary $1 \leq \ell \leq L$, and
suppose that both items above hold for $\ell-1$. Consider an arbitrary $1 \leq
i \leq \ddim[\ell]$. We first show the first item. Observe that, the definition
of $\Se{\ell}{i}$ ensures that its first element is $\act(z)$ for $z =
\Val{\ell}{i}{\s{\ell-1}}{\Ys_{\emptyset}}$. Recall that $\s{\ell-1}$ is
defined as the vector of dimension $\ddim[\ell-1]$ where $\elt{\s{\ell-1}}{j}$
is the first element of $\Se{\ell-1}{j}$, for each $1 \leq j \leq
\ddim[\ell-1]$. However, by induction hypothesis, $\elt{\s{\ell-1}}{j} =
\Next{\ell-1}{j}{\symstart}$, and so $\s{\ell-1} = \Start{\ell}$. Then, lines
\ref{alg:next:init:z} and \ref{alg:next:first} of Algorithm \ref{alg:next}
ensure that $\Next{\ell}{i}{\symstart}$ is precisely $\act(z)$. We now show the
second item. Consider an arbitrary $\alpha \in \real$. We study the execution
of $\Next{\ell}{i}{\alpha}$. Since $\alpha \neq \symstart$, $F$ is initialised
as stated in line \ref{alg:next:init:frontier} and so the loop starting in line
\ref{alg:next:loop} is executed. We consider now the outcome of the loop's
execution. Let $\Se{\ell}{i}= q_0, q_1, \dots$. Let $N \geq 0$ be the smallest
natural number such that either $q_N$ is not defined or $q_N > \alpha$; such
$N$ must exist since $\Se{\ell}{i}$ is either finite or it converges to
infinity, and furthermore it is strictly monotonically increasing. We next show
the following claim ($\ast$): for each $0 \leq n \leq N$, the algorithm's loop
reaches a state where $F = f_{n}$ after a finite number of iterations. We prove
this by induction on $n$.

The base case is straightforward since $F$ initially contains only the triple
$\triple{\Start{\ell}}{\Ys_{\emptyset}}{z}$, where $z =
\Val{\ell}{i}{\Start{\ell}}{\Ys_{\emptyset}}$. Furthermore, $f_0$ contains only
the triple $\triple{\s{\ell-1}}{\Ys_{\emptyset}}{z'}$, for $z' =
\Val{\ell}{i}{\s{\ell-1}}{\Ys_{\emptyset}}$. However, we have already shown
that $\Start{\ell}=\s{\ell-1}$, so the initial state of $F$ is identical to
$f_0$. For the induction step, consider an arbitrary $0 \leq n < N$ and suppose
that $F =f_{n}$ after a finite number of iterations of the algorithm's loop; we
then show that $F =f_{n+1}$ holds after a finite number of additional
iterations. By definition, $f_n$ contains (at least) a triple of the form
$\triple{\x}{\Ys}{z}$ with $\act(z) = q_n$, and all other triples in $f_n$ are
of the form $\triple{\x'}{\Yps}{z'}$ with $z' \geq z$. The condition in line
\ref{alg:next:mintriple} then ensures that one of the triples of the form
$\triple{\x}{\Ys}{z}$ with $\act(z) = q_n$ will be selected; since $n < N$ and
so $q_n \leq \alpha$, the condition in line \ref{alg:next:return} will not be
satisfied, so the algorithm will not exit the loop and will afterwards start a
new loop iteration. Then, the condition in line \ref{alg:next:x:add} ensures
that no triple $\triple{\x'}{\Yps}{z'}$ with $z' \leq z$ is added to $F$. Let
$K$ be the number of triples in $f_n$ of the form $\triple{\x}{\Ys}{z}$ with
$\act(z) = q_n$. We then have that after reaching the state where $F = f_n$,
the algorithm's loop will run (at least) $K$ additional times. Looking at lines
\ref{alg:next:x:expand} to \ref{alg:next:c:numbers:add}, it is clear that each
iteration removes from $F$ one of the $K$ triples and adds to $F$ all of the
triple's successors of the form $\triple{\x'}{\Yps}{z'}$ with $z' > z$. Thus,
after those $K$ additional steps, $F$ will be exactly $f_{n+1}$. This concludes
the proof of ($\ast$).

Suppose now that $\Ssub{\ell}{i}{\alpha}$ is empty, which means that all
elements of $\Se{\ell}{i}$ are smaller than $\alpha$. Then, $N$ is precisely
the number of elements of $\Se{\ell}{i}$ plus $1$, that is, $N>0$ and $q_{N-1}$
is the last defined element of $\Se{\ell}{i}$. By the claim ($\ast$), in the
execution of $\Next{\ell}{i}{\alpha}$, $F$ becomes equal to $f_{N}$ after a
finite number of steps. Since $q_N$ is undefined, $f_{N}$ must be empty. But
then, since $F=f_{N}$, the condition in line \ref{alg:next:loop} ensures that
the loop is skipped, and line \ref{alg:next:return:final} ensures that the
algorithm outputs $\symend$, as expected. If $\Ssub{\ell}{i}{\alpha}$ is not
empty, then there exists an element $\ssub{\ell}{i}{\alpha}$ satisfying
conditions \ref{cond:A:smaller}, and \ref{cond:A:bigger}. In particular, the
definition of $N$, the condition \ref{cond:A:smaller}, and the fact that
$\ssub{\ell}{i}{\alpha} > \alpha$ together ensure that $\ssub{\ell}{i}{\alpha}$
is precisely the $N$th element of $\Se{\ell}{i}$. Claim ($\ast$) ensures that
in the execution of $\Next{\ell}{i}{\alpha}$, $F$ becomes equal to $f_{N}$
after a finite number of steps. Since the $N$th element of $\Se{\ell}{i}$ is
defined, there exists a triple $\triple{x}{\Ys}{z} \in f_n$ with $\act(z) =
\ssub{\ell}{i}{\alpha}$ and every other triple $\triple{x'}{\Yps}{z'} \in F$ is
such that $z' \geq z$. But then, since $f_n=F$, the next iteration of the loop
must select a triple with $z$ as the third component (note that this triple may
not be $\triple{x}{\Ys}{z}$). But since $\act(z) =\ssub{\ell}{i}{\alpha} >
\alpha$, the test in line \ref{alg:next:return} succeeds and so the algorithm
returns $\ssub{\ell}{i}{\alpha}$, as expected. This completes the proof of the
second item.

Finally, to see that Algorithm \ref{alg:capacity} is terminating, we simply
observe that the smallest positive number in each $\X{\ell}{i}$ can be obtained
by calling $\Next{\ell}{i}{\symstart}$ and then, if this returns $0$, calling
$\Next{\ell}{i}{0}$. We have already shown that such calls terminate and return
the expected result. All other elements defined in the pseudocode of Algorithm
\ref{alg:capacity} are easily computable from the parameters of $\GNN$, and so
the algorithm terminates.
\end{proof}

    \section{Proofs for Section 5}\label{sex:proofs-max}

\equivalencemax*

\begin{proof}
We show that $\GNN$ and $\Prog_\GNN$ are equivalent by taking an arbitrary
$(\col,\ddim)$-dataset $D$ and showing $T_{\Prog_\GNN}(D) = T_\GNN(D)$.
Inclusion ${T_{\Prog_\GNN}(D) \subseteq T_\GNN(D)}$ holds because, by
definition, $T_{\GNN}$ captures each rule ${r \in \Prog_\GNN}$, which implies
${T_r(D) \subseteq T_\GNN(D)}$. Since ${T_{\Prog_\GNN}(D) = \bigcup_{r \in
\Prog_\GNN} T_r(D)}$, we have ${T_{\Prog_\GNN}(D) \subseteq T_\GNN(D)}$.

For the converse inclusion, consider an arbitrary fact ${\alpha \in
T_\GNN(D)}$. Since $\GNN$ is a max ($\col,\ddim$)-GNN, its capacity
$\capacity{\GNN}$ is bounded by $1$. The procedure in the proof of
Theorem~\ref{thm:prog:equivalence:max} therefore constructs a ($L,|\col| \cdot
\ddim[\GNN]$)-tree-like rule $r$ that is captured by $\GNN$ satisfying ${\alpha
\in T_r(D)}$. Furthermore, since ${\capacity{\ell} \leq 1}$ for each ${1 \leq
\ell \leq L}$, in equation \eqref{eq:append-gamma} in the construction of $r$
we have ${|W| \leq \capacity{\ell'} \leq 1}$, so the construction does not
introduce any inequalities in the body of $r$, and so ${r \in \Prog_\GNN}$
holds. Hence, we have ${\alpha \in T_{\Prog_\GNN}(D)}$, and so ${T_\GNN(D)
\subseteq T_{\Prog_\GNN}(D)}$, as required. \end{proof}

\maxequivalence*

\begin{proof}
For an arbitrary $(\col,\ddim)$-dataset $D$, let ${\gG = \langle \gV,
\{\gE{c}\}_{c \in \col}, \lab \rangle}$ be the canonical encoding of $D$, and
consider applying $\GNN$ to $\gG$. We prove the claim by induction over ${1
\leq \ell < L}$. For the base case ${\ell = 1}$, consider an arbitrary term
$t$, an arbitrary position ${1 \leq i \leq \ddim[1]}$, and let $v$ be the
vertex corresponding to $t$. Let $\J{1}$ and $\Jp{1}$ be the following sets of
indices.
\begin{align}
    \J{1} = \;  & \{ j \mid 1 \leq j \leq \ddim[0] \text{ and } \elt{\vlab{v}{0}}{j} = 1 \} \\
    \Jp{1} = \; & \{ j \mid 1 \leq j \leq \ddim[0] \text{ and } \elt{\matA{1}}{i,j} = 1 \}
\end{align}
Recall that ${\elt{\vlab{v}{0}}{j} \in \{ 0,1 \}}$ and ${\elt{\matA{1}}{i,j}
\in \{ 0,1 \}}$ for each $1 \leq j \leq \ddim[0]$; furthermore, $\matB{1}{c}$
has all elements equal to $0$ for each $c \in \col$, and one can check that
$\elt{\bias{1}}{i} = 1-|\Jp{1}|$. Thus, the argument of $\act$ in the
computation of $\elt{\vlab{v}{\lab[1]}}{i}$ is equal to
\begin{equation}
    |\J{1} \cap \Jp{1}| + 1 - |\Jp{1}|, \label{eq:z:layer1}
\end{equation}
which is equal to 1 if ${\Jp{1} \subseteq \J{1}}$, and otherwise it less than
or equal to $0$. Hence, ${\elt{\vlab{v}{0}}{i} = 1}$ if $\Jp{1} \subseteq
\J{1}$, and otherwise $\elt{\vlab{v}{0}}{i} = 0$. Thus, to prove the claim, we
show that $\Jp{1} \subseteq \J{1}$ if and only if ${D \models \tau_i\nu}$ holds
for ${\nu = \{ x \mapsto t \}}$. For the $(\Leftarrow)$ direction, assume that
${D \models \tau_i\nu}$ holds for ${\nu = \{ x \mapsto t \}}$, and consider an
arbitrary ${j \in \Jp{j}}$. The definition of $\Jp{1}$ implies
$\elt{\matA{1}}{i,j} = 1$, so $\tau_i$ contains $U_j(x)$. But then, ${D \models
\tau_i\nu}$ implies $U_j(t) \in D$, so our encoding ensures
${\elt{\vlab{v}{0}}{j} = 1}$; hence, ${j \in J}$ holds, as required. For the
$(\Rightarrow)$ direction, assume that ${\Jp{1} \subseteq \J{1}}$ holds. Then,
for each ${U_j(x) \in \tau_i}$, we have ${j \in \J{1}}$ and so ${U_j(t) \in
D}$. Hence, ${D \models \tau_i\nu}$ holds for ${\nu = \{ x \mapsto t \}}$.

For the induction step, consider ${1 < \ell < L}$ such that the claim holds for
${\ell - 1}$, an arbitrary term $t$, an arbitrary position ${1 \leq i \leq
\ddim[\ell]}$, and let $v$ be the vertex corresponding to $t$. We consider two
cases. The first case is ${1 \leq i \leq \ddim[\ell-1]}$; then,
${\elt{\matA{\ell}}{i,j} = 1}$ if and only if ${j = i}$, for each $j$ we have
${\elt{\matB{\ell}{c}}{i,j} = 0}$, and ${\elt{\bias{\ell}}{i} = 0}$; hence, we
have $\elt{\vlab{v}{\ell}}{i}= \elt{\vlab{v}{\ell-1}}{i}$, so both properties
hold by the induction hypothesis. The second case is ${\ddim[\ell-1] < i \leq
\ddim[\ell]}$. For each ${c \in \col}$, let $\J{\ell,c}$ and $\Jp{\ell,c}$ be
defined as follows.
\begin{align}
    \J{\ell,c} = \;     & \{ j \mid 1 \leq j \leq \ddim[\ell-1] \text{ and there exists a vertex } u \text{ such that } \langle v, u \rangle \in \gE{c} \text{ and } \elt{\vlab{u}{\ell-1}}{j} = 1 \} \\
    \Jp{\ell,c} = \;    & \{ j \mid 1 \leq j \leq \ddim[\ell-1] \text{ and } \elt{\matB{\ell}{c}}{i,j} = 1 \}
\end{align}
Let $\tau_i$ be of the form \eqref{eq:tau:i}. Since $\varphi_{i,0}$ is a
conjunction of atoms of the form $U(x)$, there exists some ${1 \leq j_0 \leq
\ddim[\ell-1]}$ such that ${\varphi_{i,0} = \tau_{j_0}}$. Furthermore, recall
that ${\elt{\matA{\ell}}{i,j} \in \{ 0, 1 \}}$ and ${\elt{\vlab{v}{\ell-1}}{j}
\in \{ 0, 1 \}}$ for each ${1 \leq j \leq \ddim[\ell-1]}$,
${\elt{\matB{\ell}{c}}{i,j} \in \{ 0, 1 \}}$ for all ${c \in \col}$, and
${\elt{\bias{\ell}}{i} = - \sum_{c \in \col} |\Jp{\ell,c}|}$. Thus, the
argument of $\act$ in the computation of $\elt{\vlab{v}{\lab[\ell]}}{i}$ is
equal to
\begin{equation}
    \elt{\vlab{v}{\ell-1}}{j_0} + \sum_{c \in \col} \left(|\J{\ell,c} \cap \Jp{\ell,c}| - |\Jp{\ell,c}| \right),    \label{eq:z:layer:ell}
\end{equation}
which is equal to $1$ if ${\elt{\vlab{v}{\ell-1}}{j_0} = 1}$ and ${\Jp{\ell,c}
\subseteq \J{\ell,c}}$ for each ${c \in \col}$, and otherwise it is less than
or equal to $0$. Consequently, ${\elt{\vlab{v}{\ell}}{i} = 1}$ if
${\elt{\vlab{v}{\ell-1}}{j_0} = 1}$ and ${\Jp{\ell,c} \subseteq \J{\ell,c}}$
for each ${c \in \col}$, and otherwise ${\elt{\vlab{v}{\ell}}{i} = 0}$. Thus,
to prove the claim, we show that ${\elt{\vlab{v}{\ell-1}}{j_0} = 1}$ and
${\Jp{\ell,c} \subseteq \J{\ell,c}}$ for each ${c \in \col}$ if and only if
there exists a substitution $\nu$ mapping $x$ to $t$ such that ${D \models
\tau_i\nu}$.

For the $(\Leftarrow)$ direction, assume that such substitution $\nu$ exists.
We then have $D \models \varphi_{i,0}\nu$; however, ${\varphi_{i,0} =
\tau_{j_0}}$ and $j_0 \leq \ddim[\ell-1]$, so the induction hypothesis implies
${\elt{\vlab{v}{\ell-1}}{j_0} = 1}$. To prove ${\Jp{\ell,c} \subseteq
\J{\ell,c}}$ for each ${c \in \col}$, we consider arbitrary ${c \in \col}$ and
${j_k \in \Jp{\ell,c}}$, and we let ${s = \nu(y_k)}$. Note that $D \models
E^{c}(x,y) \nu$ and so $\langle v, v_s \rangle \in \gE{c}$. Furthermore,
$\varphi_{i,k}$ is a $(\ell-2,c)$-tree-like formula for $y_k$ equal to
$\tau_{j_k}$ up to variable renaming. Also, $D \models \varphi_{i,k}\nu$
ensures that there exists a substitution $\nu_k$ mapping $x$ to $s$ such that
${D \models \tau_{j_k}\nu_k}$, so, by applying the induction hypothesis to the
vertex $u$ for term $s$, we have that ${\elt{\vlab{u}{\ell-1}}{j_k} = 1}$.
Consequently, ${j \in \J{\ell,c}}$ holds, as required.

For the $(\Rightarrow)$ direction, assume that ${\elt{\vlab{v}{\ell-1}}{j_0} =
1}$ and ${\Jp{\ell,c} \subseteq \J{\ell,c}}$ for each ${c \in \col}$. Since
${\elt{\vlab{v}{\ell-1}}{j_0} = 1}$, the induction hypothesis ensures ${D
\models \varphi_{i,0}\{ x \mapsto t \}}$. Furthermore, for each ${1 \leq k \leq
m_i}$, $\varphi_{i,k}$ is a ${(\ell-2,|\col|)}$-tree-like formula, and so there
exists ${1 \leq j_k \leq \ddim[\ell-1]}$ such that $\varphi_{i,k}$ is equal to
$\tau_{j_k}$ up to variable renaming. Furthermore, ${(\matB{\ell}{c})_{i,j_k} =
1}$ and so ${j_k \in \Jp{\ell,c}}$, which in turn implies ${j_k \in
\J{\ell,c}}$. Thus, there exists vertex $u$ for a term ${s \in \tms{D}}$ such
that ${\langle v, u \rangle \in \gE{c}}$ and ${\elt{\vlab{u}{\ell-1}}{j_k} =
1}$. By the induction hypothesis, there exists a substitution $\nu_k$ mapping
$x$ to $s$ such that ${D \models \tau_{j_k}\nu_k}$. Moreover, $\tau_{j_k}$ is
equal to $\varphi_{i,k}$ up to variable renaming, so there exists a
substitution $\nu'_k$ mapping $y_k$ to $s$ such that ${D \models
\varphi_{i,k}\nu'_k}$. Note that $\varphi_{i,k}$ has no variables in common
with $\varphi_{i,k'}$ for each ${1 \leq k < k' \leq m_i}$, and none of these
formulas mention $x$, so substitution ${\nu = \{ x \mapsto t \} \cup
\bigcup_{k=1}^{m_i} \nu'_k}$ is correctly defined. Observe that ${D \models
\varphi_{i,0}\nu}$, ${D \models \varphi_{i,k}\nu}$ for each ${1 \leq k \leq
m_i}$, and ${D \models E^c(x,y_k)\nu}$ since ${\langle v, u \rangle \in
\gE{c}}$. Thus, ${D \models \tau_i\nu}$ holds, as required. \end{proof}

\progequivalencemax*

\begin{proof}
For an arbitrary $(\col,\ddim)$-dataset $D$, let ${\gG = \langle \gV,
\{\gE{c}\}_{c \in \col}, \lab \rangle}$ be the canonical encoding of $D$, and
consider applying $\GNN$ to $\gG$. Moreover, consider an arbitrary vertex ${v
\in \gV}$ for a term ${t \in \tms{D}}$, and an arbitrary position ${1 \leq i
\leq \ddim[L]}$. We show that ${U_i(t) \in T_{\GNN_{\Prog}}(D)}$ if and only if
${U_i(t) \in T_\Prog(D)}$. Towards this goal, let $\J{L}$ and $\Jp{L}$ be the
following sets of indices.
\begin{align}
    \J{L} = \;  & \{ j \mid 1 \leq j \leq \ddim[L-1] \text{ and } \elt{\vlab{v}{L-1}}{j} = 1 \} \\
    \Jp{L} = \; & \{ j \mid 1 \leq j \leq \ddim[L-1] \text{ and } \elt{\matA{L}}{i,j} = 1 \}
\end{align}
For each ${1 \leq j \leq \ddim[L-1]}$, we have ${\elt{\matA{L}}{i,j} \in \{ 0,
1 \}}$, matrices $\matB{L}{c}$ and $\bias{L}$ have all elements equal to $0$,
and Lemma~\ref{lem:max:equivalence} ensures ${\elt{\vlab{v}{L-1}}{j} \in \{ 0,
1 \}}$. Thus, the argument of $\act$ in the computation of
$\elt{\vlab{v}{\lab[L]}}{i}$ is equal to ${|\J{L} \cap \Jp{L}|}$, which is
greater than or equal to $1$ if ${\Jp{L} \cap \J{L} \neq \emptyset}$, and
smaller than or equal to $0$ otherwise. Hence, ${\elt{\vlab{v}{L}}{i} = 1}$ if
${\Jp{L} \cap \J{L} \neq \emptyset}$, and ${\elt{\vlab{v}{L}}{i} = 0}$
otherwise.

Now assume that ${U_i(t) \in T_{\GNN_{\Prog}}(D)}$ holds. The latter implies
${\cls(\elt{\vlab{v}{L}}{i}) = 1}$, which implies ${\elt{\vlab{v}{L}}{i} \geq
1}$; moreover, as shown in the previous paragraph, then ${\Jp{L} \cap \J{L}
\neq \emptyset}$. Consider an arbitrary ${j \in \Jp{L} \cap \J{L}}$. Since ${j
\in \Jp{L}}$, there exists a rule of the form ${U_i(x) \gets \varphi \in
\Prog}$ where $\varphi$ is equal to $\tau_j$. Furthermore, ${j \in \J{L}}$
implies ${\elt{\vlab{v}{L-1}}{j} = 1}$, and by Lemma \ref{lem:max:equivalence}
there exists a substitution mapping $x$ to $t$ such that ${D \models
\varphi\nu}$. Hence, ${U_i(x)\nu \in T_\Prog(D)}$, and so ${U_i(t) \in
T_\Prog(D)}$ holds, as required.

Conversely, assume that ${U_i(t) \in T_\Prog(D)}$ holds. Fact $U_i(t)$ is
produced by a rule ${\varphi \rightarrow U_i(x) \in \Prog}$ and a substitution
$\nu$ mapping $x$ to $t$ such that ${D \models \varphi\nu}$. Since $\varphi$ is
a $(L-2,f)$-tree-like formula for $x$, there exists ${1 \leq j \leq
\ddim[L-1]}$ such that $\varphi$ is equal to $\tau_j$ up to variable renaming,
and $\tau_j$ is a $(L-2,f)$-tree-like formula for $x$. Hence, Lemma
\ref{lem:max:equivalence} ensures that ${\elt{\vlab{v}{L-1}}{j} = 1}$ and so
${j \in \J{L}}$. Furthermore, the definition of $\matA{L}$ ensures that
${\elt{\matA{L}}{i,j} = 1}$, and so $j \in \Jp{L}$. Thus, ${\Jp{L} \cap \J{L}
\neq \emptyset}$, which implies ${\elt{\vlab{v}{L}}{i} = 1}$; this, in turn,
ensures ${\cls(\elt{\vlab{v}{L}}{i}) = 1}$, so $U_i(t) \in T_{\GNN_{\Prog}}(D)$
holds, as required.

We next provide an upper bound on $\ddim[L-1]$. By
Definition~\ref{def:tree-like}, the fan-out of a variable of depth $i$ is at
most $f(d-i)$; moreover, the number of variables of depth $i$ is at most the
number of variables of depth $i-1$ times the fan-out of each variable, which is
${f^i \cdot d \dots (d - i + 1)}$ and is bounded by ${f^i \cdot d!}$. By adding
up the contribution of each depth, there are at most ${f^d \cdot (d+1)!}$
variables. Each variable is labelled by one of the ${2^{\ddim}}$ formulas of
depth zero, and each non-root variable is connected by one of the $|\col|$
predicates to its parent. Hence, there are at most ${(|\col| \cdot
2^{\ddim})^{f^d \cdot (d+1)!}}$ tree-like formulas.
\end{proof}

}{}

\end{document}